\documentclass{article}

\usepackage{microtype}
\usepackage{graphicx}
\usepackage{booktabs} 
\usepackage{arydshln}
\usepackage[algo2e]{algorithm2e} 
\usepackage{algorithm,algorithmic}

\usepackage{hyperref}
\usepackage[accepted]{icml2021}

\usepackage{amsmath,amsfonts,bm}

\def\1{\bm{1}}

\def\vtheta{{\bm{\theta}}}
\def\vphi{{\bm{\phi}}}
\def\vh{{\bm{h}}}

\def\vb{{\bm{b}}}

\def\vd{{\bm{d}}}

\def\vf{{\bm{f}}}

\def\vh{{\bm{h}}}

\def\mA{{\bm{A}}}

\def\mD{{\bm{D}}}

\def\mF{{\bm{F}}}

\def\mI{{\bm{I}}}

\def\mN{{\bm{N}}}

\def\mP{{\bm{P}}}

\def\mPhi{{\bm{\Phi}}}

\def\mGamma{{\bm{\Gamma}}}

\DeclareMathAlphabet{\mathsfit}{\encodingdefault}{\sfdefault}{m}{sl}
\SetMathAlphabet{\mathsfit}{bold}{\encodingdefault}{\sfdefault}{bx}{n}

\newcommand{\E}{\mathbb{E}}

\newcommand{\sigmoid}{\sigma}

\usepackage{amsthm}
\newtheorem{theorem}{Theorem}

\newtheorem{proposition}[theorem]{Proposition}

\usepackage{tikz}
\usetikzlibrary{arrows,shapes,fit,shadows,decorations.pathmorphing}
\usepackage{subfigure,wrapfig}

\tikzstyle{circ}=[circle,draw=black,thick,minimum size=8mm,>=stealth] 
\tikzstyle{state}=[circle,draw=black,minimum size=2em,>=stealth] 
\tikzstyle{ellip}=[ellipse,draw=black,thick,minimum size=4mm, text width=.4cm]
\tikzstyle{system} = [draw, dotted, minimum height=2em]
\tikzstyle{input} = [text centered, minimum height=2em]
\tikzstyle{null} = [inner sep=0, outer sep=0]
\tikzstyle{op} = [draw, fill=blue!20, text centered,
                  minimum height=2em, rounded corners]
\tikzstyle{mlp} = [draw, text width=10em, fill=blue!20, text centered,
                  minimum height=2em, rounded corners]
\tikzstyle{inp} = [text centered, minimum height=1em]
\tikzstyle{dgraph}=[->, line width=1pt]

\def\w{\mathbf{w}}
\def\e{\mathbf{e}}
\usepackage[capitalize]{cleveref}
\usepackage{cancel}

\icmltitlerunning{Emphatic Algorithms for Deep Reinforcement Learning}

\begin{document}

\twocolumn[
\icmltitle{Emphatic Algorithms for Deep Reinforcement Learning}

\begin{icmlauthorlist}
\icmlauthor{Ray Jiang}{dm}
\icmlauthor{Tom Zahavy}{dm}
\icmlauthor{Zhongwen Xu}{dm}
\icmlauthor{Adam White}{dm,ua}\\
\icmlauthor{Matteo Hessel}{dm}
\icmlauthor{Charles Blundell}{dm}
\icmlauthor{Hado van Hasselt}{dm}
\end{icmlauthorlist}

\icmlaffiliation{dm}{DeepMind, London, UK.}
\icmlaffiliation{ua}{Amii, Department of Computing Science,
University of Alberta}

\icmlcorrespondingauthor{Ray Jiang}{rayjiang@google.com}

\icmlkeywords{Machine Learning, ICML, emphatic, reinforcement learning}

\vskip 0.3in
]



\printAffiliationsAndNotice{}  

\begin{abstract}
Off-policy learning allows us to learn about possible policies of behavior from experience generated by a different behavior policy. Temporal difference (TD) learning algorithms can become unstable when combined with function approximation and off-policy sampling---this is known as the ``deadly triad''. Emphatic temporal difference (ETD($\lambda$)) algorithm ensures convergence in the linear case by appropriately weighting the TD($\lambda$) updates. In this paper, we extend the use of emphatic methods to deep reinforcement learning agents. We show that naively adapting ETD($\lambda$) to popular deep reinforcement learning algorithms, which use forward view multi-step returns, results in poor performance.
We then derive new emphatic algorithms for use in the context of such algorithms, and we demonstrate that they provide noticeable benefits in small problems designed to highlight the instability of TD methods.
Finally, we observed improved performance when applying these algorithms at scale on classic Atari games from the Arcade Learning Environment.
\end{abstract}

\label{introduction}

Off-policy learning, whereby an agent learns from behavior that differs from its current policy, affords an agent opportunities to accumulate rich knowledge \citep{degris2012knowledge} by learning about the effect of different policies of behaviors. This can also be extended to learn about different goals, e.g., by learning \emph{general value functions} \citep{sutton2011horde} for cumulants that differ from the main task reward. Unfortunately, it is well known that reinforcement learning algorithms \citep{sutton2018} can become unstable when combining function approximation, off-policy learning, and bootstrapping \citep{Tsitsiklis_VanRoy:97}---for this reason such combination is referred to as the \emph{deadly triad} \citep{sutton2018,hasselt2018}.

Many reinforcement learning (RL) agents learn off-policy to some extent, to learn about the greedy policy while exploring \citep{watkins1989}, to make predictions about policies simultaneously \citep{sutton2011horde, zahavy2020, jaderberg2016}, to improve sample complexity via experience replay \citep{Lin:1992,mnih2015}, or even just to correct for the latency introduced by distributed computation \citep{espeholt2018}. Since these algorithms make use of bootstrapping and function approximation, they may suffer from deadly triad symptoms of ``soft divergence'' and slow convergence \citep{hasselt2018}.

The ETD($\lambda$) algorithm \citep{sutton2016emphatic} ensures convergence with linear function approximation \citep{yu2015} by weighting the updates of TD($\lambda$) (in the backward view, with eligibility traces). However, combining eligibility traces and deep neural networks can be challenging \citep{sutton1989}\footnote{See \citet{vanHasselt:2021} for recent developments.}, and thus widely used deep RL systems instead typically use $n$-step forward view methods. Overall, none of the existing solutions to the deadly triad \citep{Sutton:2009,sutton2016emphatic} have become standard practice in deep RL. 
In this paper, we extend the emphatic method to multi-step deep RL learning targets, including an off-policy value-learning method known as `V-trace' \citep{espeholt2018} that is often used in actor-critic systems.

The structure of this paper is the following. Sec.~\ref{sec:background} explains the background on forward view learning targets and ETD($\lambda$). Next we adapt ETD($\lambda$) to the forward view in Sec.~\ref{sec:wetd}, and derive a new multi-step emphatic trace for $n$-step TD in Sec.~\ref{sec:netd}. We discuss further algorithmic considerations in Sec.~\ref{sec:variants}, including extensions for variance reduction, for the V-trace value learning target and for the actor critic learning algorithms. Empirically, we provide an in-depth comparison of their qualitative properties on small diagnostic MDPs in Sec.~\ref{sec:mdp_exp}. Finally, we demonstrate that combining emphatic trace with deep neural networks can improve performance on classic Atari video games in Sec.~\ref{sec:atari_exp}, reporting the highest score to date for an RL agent without experience replay in the 200M frames data regime: $497\%$ median human normalized score across 57 games, improved from the baseline performance of 403\%.

\section{Background}
\label{sec:background}
A Markov decision process \citep[MDP; ][]{bellman1957} consists of finite sets of states $\mathcal{S}$ and actions $\mathcal{A}$, a reward function $r: \mathcal{S}\times \mathcal{A} \mapsto \mathbb{R}$, a transition distribution $P(s'|s, a)$ $s,s'\in\mathcal{S},a \in \mathcal{A}$, and a discount factor $\gamma$. A policy is a distribution over actions conditioned on the state: $\pi(a | s)$. The goal of RL  is to find a policy $\pi$ that maximizes the expected discounted \emph{return} $v_{\pi}(s) \doteq \E_{\pi}\left[\sum_{t\geq 0} (\prod_{i=1}^t \gamma_t) R_{t+1} \right]$ where, at time $t$, $\gamma_t$ denotes the scalar discount, $S_t \in \mathcal{S}$ the state variable, $A_t\in\mathcal{A}$ the action taken and $R_{t+1} \doteq r(S_t, A_t)$ the reward.\footnote{We use the notation ``$\doteq$'' to indicate an equality by definition rather than by derivation.}

\subsection{TD($\lambda$)}

Policy \emph{evaluation} is the problem of learning to predict the value $V_{\vtheta}(s) \approx v_{\pi}(s)$, for all states $s$, under an arbitrary (fixed) policy $\pi$ and parametrized by $\vtheta$. When using function approximation, each state $S_t$ is associated with a feature vector $\vphi_t$, and the agent's value estimates $V_{\vtheta}(s)$ are a parametric function of these features. TD($\lambda$) \citep{Sutton:1988} is a widely used algorithm for policy evaluation where, on each step $t$, the parameters of $V_{\vtheta}$ are updated according to
\[
\vtheta_{t+1} \doteq \vtheta_t + \alpha_t \delta_t \e_t \,,
\]
where $\e_t = \gamma_t \lambda \e_{t-1} + \nabla_{\vtheta} V_{\vtheta}(S_t)$ is an \emph{eligibility trace} of value gradients, $\delta_t = R_{t+1} + \gamma v_{\vtheta}(S_{t+1}) - V_{\vtheta}(S_t)$ is the \emph{temporal difference (TD) error}, and $\alpha_t \in [0, 1]$ is the step-size. With linear function approximation $V_{\vtheta}(t) = \vtheta^\top \vphi_t$, the gradient $\nabla_{\vtheta}V_{\vtheta}(S_t)$ is the state features $\vphi_t$. TD($\lambda$) uses \emph{bootstrapping}, where the agent's own value estimates $V_{\vtheta}(S_t)$ are used to update the values online, on each step, without waiting for the episodes to fully resolve. TD algorithms can also be used to learn policies (i.e. for \emph{control}), by using similar updates to learn action values, or by combining value learning with policy gradients in actor-critic systems \citep{Sutton:2000}.

Temporal difference algorithms can be extended to policy evaluation (or control) in \emph{off-policy} settings, where the agent learns predictions about a \emph{target} policy $\pi$, from trajectories $(S_i, A_i, R_{i+1})_{i=t}^{t+n}$ sampled under a different \emph{behavior} policy $\mu$. However, when combining function approximation with bootstrapping and off-policy learning, the parameters may diverge \citep{baird1995,Tsitsiklis_VanRoy:97}, a phenomenon referred to as the \emph{deadly triad}.

\subsection{Emphatic TD($\lambda$)}

\label{sec:etd}
{\em Emphatic TD($\lambda$)}~\cite{sutton2016emphatic} resolves the instability due to the deadly triad by adjusting the magnitude of updates on each time-step. The idea is to re-weight the distribution of TD($\lambda$) updates to account for the likelihood of the trajectory leading to the updated state, under the target policy. Each update is emphasized or de-emphasized by a scalar {\em follow-on trace}\footnote{The original formula $F^e_t \doteq \gamma(S_t) \rho_{t-1} F^e_{t-1} + i_t$  has an additional scalar $i_t$, indicating ``interest'' in state $S_t$. We let $i_t \doteq 1$.}:
\begin{equation}\label{eq:followon}
F_t \doteq \gamma(S_t) \rho_{t-1} F_{t-1} + 1.
\end{equation}
The
Emphatic TD($\lambda$) algorithm~\cite{sutton2016emphatic}, ETD($\lambda$) for short, incorporates $F_t$ into the conventional eligibility trace update of TD($\lambda$) by emphasizing states
\begin{equation} \label{eq:etd_lambda} e_t \doteq \rho_t \big(\gamma(S_t) \lambda(S_t) e_{t-1} + M_t \vphi_t\big)\,. \notag \end{equation}
where $\rho_t \doteq \frac{\pi(A_t|S_t)}{\mu(A_t|S_t)}$ is the {\em importance sampling (IS)} ratio for the target policy $\pi$ and the behavior policy $\mu$. The \emph{emphatic trace} $M_t$ encodes how much the current state is bootstrapped by other states based on the follow-on trace:
\begin{align}
\label{eq:etd}
    M_t = \lambda(S_t) + (1 - \lambda(S_t)) F_t \,.
\end{align}
Prior work on off-policy TD($\lambda$) corrected the state-distribution using the stationary distribution induced by the target policy \cite{precup2001off}, unlike ETD($\lambda$) which uses the distribution of states produced by starting the target policy in the stationary distribution of the behavior policy. 

Extensions to ETD($\lambda$) include the ETD($\lambda$, $\beta$) algorithm, which uses an additional hyper-parameter $\beta$ in place of the discount $\gamma$ to control variance by setting $\beta < \gamma$ \cite{hallak2016}, and the ACE algorithm that applies emphatic weightings to policy gradient updates \cite{imani2018}.

ETD($\lambda$) is convergent with \emph{linear} function approximation \cite{yu2015}, but its performance when combined with \emph{non-linear} function approximation has not yet been extensively evaluated.

\subsection{$n$-step TD}
In this paper, we generalize the emphatic approach to widely used deep RL systems, and in particular actor-critic systems. Contrasting with the backward view TD($\lambda$) learning target for which ETD($\lambda$) was developed, deep RL algorithms are often based on a \emph{forward view} of temporal difference learning, where updates are computed on trajectories of fixed length, without making use of eligibility traces.

If we use a linear value function parametrized by $\vtheta$, then the $n$-step TD update for parameters $\vtheta$ in the first state is
\begin{align}
\label{eq:tdn}
     \vtheta_{t+1} \doteq \vtheta_t + \alpha \sum_{i=t}^{t+n-1} (\prod_{j=t}^{i-1} \rho_j \gamma_{j+1} ) \hskip0.1cm \rho_i \delta_i(\vtheta_t) \vphi_t,
\end{align}
where
\begin{align}
\delta_i(\vtheta_t) = R_{i+1} + \gamma_{i+1}V_{\vtheta_t}(S_{i+1})-V_{\vtheta_t}(S_i) \,.
\end{align}
$n$-step TD can be implemented in a computationally efficient way where multiple states in a trajectory are updated at once. This can be done in two ways. In a {\em fixed} update scheme all states are updated with $n$ step TD target for a fixed constant $n$. In a {\em mixed} update scheme, the $k$-th sample in the trajectory uses an $(n-k)$-step TD target---this is convenient when we used a small batch of temporal data, and all returns bootstrap on the last available state at the end of this window. 

\subsection{V-trace}
Given a trajectory of data, sampled from behavior policy $\mu$, the $n$-step V-trace estimator can be used as target to learn the value of state $S_t$ under the target policy $\pi$. Let $G_t$ be the V-trace target:
\begin{align}
\label{eq:vtrace_orig}
G_t \doteq V_{\vtheta_t}(S_t) + \hskip-0.1cm \sum_{i=t}^{t+n-1} (\prod_{j=t}^{i-1} \bar{c}_j \gamma_{j+1}) \hskip0.1cm \bar{\rho_i} \delta_i(\vtheta_t),
\end{align}
where $\bar{\rho}_i \doteq \min(\bar{\rho}, \frac{\pi(A_i|S_i)}{\mu(A_i|S_i)})$, $\bar{c}_j \doteq \min(\bar{c}, \frac{\pi(A_j|S_j)}{\mu(A_j|S_j)})$. The clipping hyper-parameters $\bar{\rho}$ and $\bar{c}$ were introduced to reduce variance of the $n$-step off-policy TD target. In practice, the clipping thresholds $\bar{c}$ and $\bar{\rho}$ are often equal, so that $\bar{c}_t = \bar{\rho}_t$.

Modifying $\bar{c}$ does not change the fixed point of the (tabular) V-trace update (see the proof of the V-trace fixed point in Appendix A of \citet{espeholt2018}, and see also \citet{Mahmood:2017AB}). Modifying $\bar{\rho}$ does change the fixed point, which corresponds to the value of the following policy $\pi_{\bar{\rho}}$: 
\begin{align}
\label{eq:vtrace_fp}
    \pi_{\bar{\rho}}(a|s) &\doteq \frac{\min (\bar{\rho}\mu(a|s), \pi(a|s))}{\sum_{a'\in \mathcal{A}}\min (\bar{\rho}\mu(a'|s), \pi(a'|s))}.
\end{align}
With linear functions the V-trace update closely matches \eqref{eq:tdn}, except in clipping all IS weights. 

\subsection{Actor-critics}

The V-trace update is most often used in actor-critic systems. Here, in addition to using it for learning values (the \emph{critic}) we can use the V-trace target also in the policy update.

Consider a current policy $\pi_{\w}$. Following the derivation of policy gradient in \citet{espeholt2018}, we may update policy parameters $\w$ in the direction of the policy gradient
\begin{align}
    \bar{\rho}_t (R_{t+1} + \gamma_{t+1} G_{t+1} - V_{\vtheta}(S_t)) \nabla_{\w}\log \pi_{\w}(A_t| S_t) \,,
\end{align}
where $G_{t+1}$ is the V-trace value target from time step $t+1$ onward. This has been very successful \citep[e.g.,][]{espeholt2018,hessel2019multi} in setting where the off-policyness is mild.

\section{Proposed Emphatic Methods}
Similar to TD($\lambda$), off-policy $n$-step TD can suffer from unstable learning due to the deadly triad. In the appendix, we analyze the update and derive conditions that guarantee stable learning when the behavior policy is sufficiently close to the target policy. However, these conditions are often violated in practice when the policies are too different. Then emphatic methods could help stabilize learning by mitigating the mismatch in steady-state distributions under the target and behavior policies. Therefore, we now first adapt ETD($\lambda$) to make use of $n$-step targets and analyze its properties, and then introduce new updates that combine emphatic methods with off-policy targets based on V-trace.

\subsection{Windowed ETD($\lambda$) --- WETD}
\label{sec:wetd}
As described in Sec.~\ref{sec:etd}, ETD($\lambda$) uses TD($\lambda$) as its learning target. To extend this idea, we adapt ETD($\lambda$) to use update windows of length $n$. Each state in the window is updated with a variable bootstrap length, all bootstrapping on the last state in the window --- this is the mixed update scheme. We formulate the learning target TD($\lambda$) as a {\em mixed $n$-step target} by setting $\lambda_t$ to 0 every $n$ steps:
\begin{align} 
\label{eq:tdn_lambda}
    \lambda_t \doteq \begin{cases}
                    0, & \text{if $t\bmod{n}=0$}.\\
                    1, & \text{otherwise}.
                \end{cases}
\end{align}
The bootstrapping step of the update is at the nearest future time step that is a multiple of the window size $n$. We set $\lambda$ to 0 at the end of every update window canceling all future TD errors from that point onward. More details are provided in the appendix.
ETD($\lambda$), as in \eqref{eq:etd_lambda}, was originally derived for state-dependent $\lambda(S_t)$. To apply the same derivation to a time-dependent $\lambda_t$, we note that under mild assumptions (that state visits are non-periodic), the value of $\lambda_t$ in \eqref{eq:tdn_lambda} is statistically independent of the state. This means the expected updates are asymptotically equivalent to using a uniform $\lambda = \mathbb{E}_{d_\mu}[ \lambda_t ] = 1 - 1/n$.

The {\em windowed ETD($\lambda$)} (WETD) algorithm is then defined by using $\lambda_t$ from \eqref{eq:tdn_lambda} in the ETD($\lambda$) update in \eqref{eq:etd}, so that
\begin{align}
\label{eq:wetd}
    M^w_t &\doteq \lambda_t + (1 - \lambda_t) F_t\,,
\end{align}
with \eqref{eq:followon} and \eqref{eq:tdn_lambda}. The WETD-corrected $n$-step TD target is obtained by multiplying $M^w_t$ to weight each update:
\begin{align}
     \vtheta_{t+1} \doteq \vtheta_t + \alpha M^w_t\sum_{i=t}^{t+n-1} (\prod_{j=t}^{i-1} \gamma_{j+1}\rho_j ) \hskip0.1cm \rho_i \delta_i(\vtheta_t) \vphi_t\,.
\end{align}
The algorithm is shown below in Algo.~\ref{alg:wetd}.

\begin{algorithm}[t]
\textbf{Input: Target policy $\pi$, behavior policy $\mu$, bootstrapping length $n$, gradient step size $\alpha$, discounts $\gamma_t$, state features $\vphi_t$.}\\
Initialize model parameters $\vtheta_0$.\\
Initialize $F_0=1$.\\
\vspace{0.08cm}
\For{$t\in[0, n, 2n, \ldots, Ln]$}{
  Sample trajectory $(S_i, A_i, R_{i+1})_{i=t}^{t+n} \sim \mu$.\\
  Set $\rho_i = \pi(A_i|S_i)/\mu(A_i|S_i)$, for $i=t, \ldots, t+n$.\\
  \For{$k\in[0, \ldots, n-1]$}{
      Compute $F_{t+k+1} = \gamma_{t+k+1} \rho_{t+k} F_{t+k} + 1$.\\
      \textbf{if} $k=0$ \textbf{then} $M^w_{t+k} = F_{t+k}$ \textbf{else} $M^w_{t+k} = 1$ \textbf{end}\\
      Update model parameters: \\
      $\vtheta_{t+k+1} = \vtheta_{t+k} + $\\
      $\alpha M^w_{t+k}\sum_{i=t+k}^{t+n-1} (\prod_{j={t+k}}^{i-1} \gamma_{j+1}\rho_j ) \hskip0.1cm \rho_i \delta_i(\vtheta_{t+k}) \vphi_{t+k}\,$.
      }
  }
\textbf{Return:} $\vtheta_{Ln}$.
\caption{WETD weighted $n$-step TD.}
\label{alg:wetd}
\end{algorithm}

\subsection{Emphatic TD($n$) --- NETD}
\label{sec:netd}
We also investigate the use of off-policy $n$-step TD target and derive from scratch a new emphatic trace called {\em Emphatic TD($n$)}, abbreviated as {\em NETD}. Similar to ETD($\lambda$), NETD guarantees asymptotic stability in off-policy learning with linear value function approximation by ensuring that the asymptotic update matrix is positive definite (see the appendix for its derivation and stability analysis). 

Consider an $n$-step TD update (in the fixed update scheme). We define the NETD trace as 
\begin{align}
\label{eq:netd}    
    F^{(n)}_t = \prod_{i=1}^{n} (\gamma_{t-i+1}\rho_{t-i}) F^{(n)}_{t-n} + 1, 
\end{align}
where $F^{(n)}_0, F^{(n)}_1, \ldots, F^{(n)}_{n-1}=1$. We can apply this new trace to weight each $n$-step TD update to $\vtheta$, i.e.
\begin{align}
     \vtheta_{t+1} = \vtheta_t + \alpha F^{(n)}_t\sum_{i=t}^{t+n-1} (\prod_{j=t}^{i-1} \gamma_{j+1}\rho_j ) \hskip0.1cm \rho_i \delta_i(\vtheta_t) \vphi_t\,.
\end{align}

NETD accumulates every $n$ steps, making it a tamer trace than the WETD follow-on trace $F_t$ (see Prop.~\ref{prop}). For a concrete example, consider $\gamma\equiv 0.99$ and in the on-policy case, $\rho\equiv1$. For WETD, the fixed point of $F_t$ is $100$, whereas the fixed point of $F^{(n)}_t$ is $1/(1-0.99^n)$, which is $10.46$ for $n=10$, $3.84$ for $n=30$, and only $1.58$ for $n=100$. As a result, NETD can be more stable than WETD when large bootstrap lengths are used, which is common in practice. 

\begin{proposition}
\label{prop}
Assume $\rho_t>0$ and $\gamma_t>0$ for any time step $t$. Then we have $F_t > F^{(n)}_t$ for any $t>0$.
\end{proposition}
\begin{proof}
We prove this result by induction for every $n$ time steps. At the start, for $t=0$, we have $F_0 = F^{(n)}_0=1$. For $0<t\leq n$, since $\rho_t>0, \gamma_t>0$ and $F_0=1$, we know $F_t = \gamma_t \rho_{t-1} F_{t-1} + 1$ is always a positive number. Thus $F_t = \gamma_t \rho_{t-1} F_{t-1} + 1 > 1 = F^{(n)}_t$ since $\gamma_t \rho_{t-1} F_{t-1} > 0$. Now assume that $F_{t-n} > F^{(n)}_{t-n}$, we derive that $F_t > F^{(n)}_t$ as follows. Substituting Eq.~\eqref{eq:followon} $n$ times, we have
\begin{align}
    F_t &= \prod_{i=1}^{n} (\gamma_{t-i+1}\rho_{t-i}) F_{t-n} \\
    & \hspace{0.9cm} + \sum_{k=1}^{n-1} \prod_{j=1}^k (\gamma_{t+1-j} \rho_{t-j}) + 1 \\
    & \geq F^{(n)}_t + \sum_{k=1}^{n-1} \prod_{j=1}^k (\gamma_{t+1-j} \rho_{t-j}) > F^{(n)}_t.
\end{align}
Since $F_t > F^{(n)}_t$ for $t=0, 1,\ldots, n$, this is also true for $t=n+1, \ldots, 2n+1$ and so on for every time step $t$.
\end{proof}
Therefore $F_t$ used in WETD is a strict upper bound for $F^{(n)}_t$ in NETD. Moreover the difference between them $\sum_{k=1}^{n-1} \prod_{j=1}^k (\gamma_{t+1-j} \rho_{t-j})$ grows with $n$. Algo.~\ref{alg:netd} shows the pseudo-code of NETD weighted TD learning.

\begin{algorithm}[t]
\textbf{Input: Target policy $\pi$, behavior policy $\mu$, bootstrapping length $n$, gradient step size $\alpha$, discounts $\gamma_t$, state features $\vphi_t$.}\\
Initialize model parameters $\vtheta_0$. \\
Initialize $F^{(n)}_0, \ldots, F^{(n)}_{n-1}=1.$ \\
Sample trajectory $(S_i, A_i, R_{i+1})_{i=0}^{n-1} \sim \mu$.\\
 Set $\rho_i = \pi(A_i|S_i)/\mu(A_i|S_i)$, for $i=1, \ldots, n-1$.\\
\For{$t\in[0, \ldots, T]$}{
  Sample $S_{t+n}, A_{t+n}, R_{t+n+1} \sim \mu$\\
  Set $\rho_{t+n} = \pi(A_{t+n}|S_{t+n})/\mu(A_{t+n}|S_{t+n})$.\\
  \textbf{if} $t\geq n$ \textbf{then}\\
  Compute $F^{(n)}_t = \prod_{i=1}^{n}(\gamma_{t-i+1} \rho_{t-i}) F^{(n)}_{t-n} + 1.$\\
  \textbf{end}\\
  Update model parameters: \\
  $\vtheta_{t+1} = \vtheta_t + \alpha F^{(n)}_t\sum_{i=t}^{t+n-1} (\prod_{j=t}^{i-1} \gamma_{j+1}\rho_j ) \hskip0.1cm \rho_i \delta_i(\vtheta_t) \vphi_t\,.$
  }
\textbf{Return:} $\vtheta_T$.
\caption{NETD weighted $n$-step TD.}
\label{alg:netd}
\end{algorithm}

\subsection{Emphatic Variants}
\label{sec:variants}
In this section we derive novel emphatic updates based on on either the WETD or the NETD trace.

\begin{table*}[h!]
\begin{center}
\begin{tabular}{lcc|cccc}
&\multicolumn{2}{c}{Emphatic Trace Computation} &\multicolumn{4}{c}{Learning Target Computation}\\
\cline{2-3} \cline{4-7}
Algorithm &  transform $\rho$  & trace type & learning target & update scheme & clip $c$ & clip $\rho$\\
\toprule
$n$-step TD$^1$         &N/A &x      &$\pi^*$ & either &x  &x\\
NETD                     &x   &NETD   &$\pi^*$ & fixed  &x  &x\\
WETD                     &x   &WETD   &$\pi^*$ & mixed  &x  &x\\
Clip-NETD                &$\min(\bar{\rho}, \rho)$ &NETD  &unknown & fixed  &x  &x\\
Clip-WETD                &$\min(\bar{\rho}, \rho)$ &WETD  &unknown & mixed  &x  &x\\
V-trace$^1$              &N/A      &x     &$\pi_{\bar{\rho}}$ & either &\checkmark &\checkmark\\ 
NEVtrace                 &$\rho^v \doteq \pi_{\bar{\rho}}/\mu$ &NETD  &$\pi_{\bar{\rho}}$ & fixed  &\checkmark &\checkmark\\
WEVtrace                 &$\rho^v \doteq \pi_{\bar{\rho}}/\mu$ &WETD  &$\pi_{\bar{\rho}}$ & mixed  &\checkmark &\checkmark\\
\bottomrule
\end{tabular}
\end{center}
\caption{Look-up table for our emphatic algorithms and the two baseline algorithms without emphatic traces. $\pi^*$ is the optimal policy for $n$-step TD learning. $\pi_{\bar{\rho}}$ is the fixed point target policy of V-trace (Eq.~\ref{eq:vtrace_fp}). When applied to Surreal (explained in Sec.~\ref{sec:atari_exp}) for large scale experiments, we always clip IS weights in computing emphatic traces to reduce variance except for NEVtrace and WEVtrace.}
\label{tab:methods}
\end{table*}

\paragraph{Clipped Emphases}
To further reduce variance of the emphatic algorithms, we clip the IS weights used in computing WETD and NETD in Eq.~\ref{eq:wetd} \& \ref{eq:netd}, and keep the IS weights used in computing the learning update unchanged. We call this new emphatic trace {\em Clip-WETD} in the case of WETD,  
\begin{align}
\label{eq:clip_wetd}
    \bar{F}_t = \bar{\rho}_{t-1} \gamma_t \bar{F}_{t-1} + 1, 
\end{align}

and {\em Clip-NETD} in the case of NETD,
\begin{align}
\label{eq:clip_netd}    
    \bar{F}^{(n)}_t = \prod_{i=1}^{n}(\gamma_{t-i+1}\bar{\rho}_{t-i}) \bar{F}^{(n)}_{t-n} + 1.
\end{align}
Note that clipping reduces the growth of emphatic traces, but may introduce bias in the emphatic trace weighted learning updates.

\paragraph{Emphatic V-trace}
We can also use the emphatic traces WETD and NETD in combination with V-trace value target. In the case of WETD, we first adapt TD($\lambda$) to the mixed V-trace learning target with windows of length $n$ by defining the new $\lambda^v_t$ as
\begin{align}
\label{eq:vtrace_lambda}
    \lambda^v_t \doteq \begin{cases}
                       0, & \text{if $t\bmod{n}=0$}.\\
                       \bar{\rho}_t/\rho_t, & \text{otherwise}.
                     \end{cases}
\end{align}
where $\bar{\rho}_t = \min(\bar{\rho}, \rho_t)$ and $\bar{\rho}$ is the clipping threshold on IS weights of the learning target. Similar to the adaption to off-policy $n$-step TD target, this way the future TD errors not only stop affecting the update if they lie beyond the current window (due to $\lambda^v$ set to 0), but also the relevant TD errors are weighted according to the clipped IS weights as in the V-trace value target. The appendix contains a detailed analysis on why this recovers the mixed V-trace learning target.
We then adapt WETD to the V-trace learning target by adopting its target policy ($\pi_{\bar{\rho}}$ in Eq.~\ref{eq:vtrace_fp}) as the target policy of the emphatic trace. The IS ratios in computing $F^v_t$ are between the V-trace target policy $\pi_{\bar{\rho}}$ and the behavior policy $\mu$, i.e. $\rho_t^v = \pi_{\bar{\rho}}(A_t | S_t)/\mu(A_t | S_t)$, and
\begin{align}
\label{eq:wevtrace}
    F^v_t = \rho^v_{t-1} \gamma_t F^v_{t-1} + 1. 
\end{align}
We call this new emphatic trace the Windowed Emphatic Vtrace, abbreviated as {\em WEVtrace}.

In the case of NETD, we similarly extend it to the fixed V-trace target by replacing the IS weights in Eq~\ref{eq:netd} by $\rho_t^v$. We call it the $N$-step Emphatic V-trace, {\em NEVtrace}. 
\begin{align}
\label{eq:nevtrace}    
    F^{(n),v}_t = \prod_{i=1}^{n}(\gamma_{t-i+1} \rho^v_{t-i}) F^{(n),v}_{t-n} + 1.
\end{align}
See the appendix for derivation details. Notice that NEVtrace and WEVtrace are likely to have higher variances than Clip-NETD and Clip-WETD (see the inequality below for any time step $t>0$). Though they are the correct emphatic traces w.r.t. to the V-trace target, in practice they often perform worse than the clipped emphatic traces when applied to the V-trace target.
\begin{align}
    \rho^v_t &= \pi_{\bar{\rho}}(A_t | S_t)/\mu(A_t | S_t) \\
    &= \frac{\min (\bar{\rho}, \pi(A_t|S_t)/\mu(A_t|S_t))}{\sum_{a'\in \mathcal{A}}\min (\bar{\rho}\mu(a'|A_t), \pi(a'|S_t))} \\
    &=  \frac{\bar{\rho}_t}{\sum_{a'\in \mathcal{A}}\min (\bar{\rho}\mu(a'|S_t), \pi(a'|S_t))} \geq  \bar{\rho}_t.
\end{align}

Table~\ref{tab:methods} lists all the emphatic algorithms and the three baseline learning algorithms with their respective variations in learning updates and emphatic trace computation. 

\paragraph{Emphatic Actor-critics}
Actor critic agents reportedly can suffer more from off-policy learning than value-based agents, which is one of the main reasons we choose to focus on V-trace in this paper. We can combine the emphatic traces derived above with the off-policy $n$-step TD or the V-trace value targets, and apply these to actor critic by simply applying emphatic traces to both the value estimate gradient and the policy gradient in for example, the V-trace learning update, following existing work on ACE \cite{imani2018}. We name these new emphatic algorithms after the emphatic trace used, which can be any of, e.g. NETD, WETD, Clip-WETD, Clip-NETD, NEVtrace, WEVtrace.  We add an `-ACE' suffix to indicate when the same emphatic trace is applied not just to the value update, but also to the policy gradient update.

\section{Diagnostic Experiments}
\label{sec:mdp_exp}

\begin{figure}[t]
\centering
\vskip -0.1cm
\begin{tikzpicture}[dgraph]
\node[circ] (s1) at (0, 1) {$\theta$};
\node[circ] (s2) at (1.5, 1) {$2\theta$};

\draw[dashed](s2.210)to[out=210, in=-30,looseness=1] (s1.-30);
\draw[dashed](s1.150)to[out=150, in=210,looseness=6] (s1.210);
\draw[dashed](s2.-15)to[out=-15, in=15,looseness=8] (s2.15);
\draw[dashed](s1)--(s2);
\draw[](s1.30)to[out=30, in=150,looseness=1] (s2.150);
\draw[](s2.-30)to[out=-30, in=30,looseness=8] (s2.30);
\end{tikzpicture}
\vspace{-0.5cm}
\caption{A simple two State MDP. The solid lines depict the (deterministic) target policy $\pi($right$|\cdot)=1$. Dashed lines denote the (more exploratory) behavior policy $\mu$. The behavior policy selects any of the actions with equal probability $\mu($right$|\cdot)=\mu($left$|\cdot)=0.5$, in all states. The rewards are zero everywhere.}
\label{fig:2state_mdp}
\end{figure}
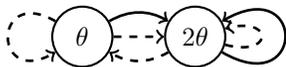

\begin{figure}[t]
\centering
\vskip -0.1cm
\resizebox{0.8\columnwidth}{!}{
\begin{tikzpicture}[dgraph]
\node[circ] (s1) at (1.6,0.8) {$S_1$};
\node[circ] (s2) at (2.9,0.8) {$S_2$};
\node[circ] (s3) at (4.2,0.8) {$S_3$};
\node[circ] (s4) at (5.5,0.8) {$S_4$};
\node[circ] (s5) at (5.5,-1) {$S_5$};
\node[circ] (s6) at (4.2,-1) {$S_6$};
\node[circ] (s7) at (2.9,-1) {$S_7$};
\node[circ] (s8) at (1.6,-1) {$S_8$};
\node[circ] (s9) at (0.3,-1) {$S_9$};

\draw[](s1.30)--(s2.150);
\draw[](s2.30)--(s3.150);
\draw[](s3.30)--(s4.150);
\draw[](s4.east)to[out=0, in=30,looseness=1] (s5.30);
\draw[](s5.150)--(s6.30);
\draw[](s6.150)--(s7.30);
\draw[](s7.150)--(s8.30);
\draw[](s8.150)--(s9.30);
\draw[](s9.-150)to[out=-150, in=150,looseness=8] (s9.150);

\draw[dashed](s1.-30)--(s2.-150);
\draw[dashed](s2.-30)--(s3.-150);
\draw[dashed](s3.-30)--(s4.-150);
\draw[dashed](s4.-30)to[out=-30, in=60,looseness=1] (s5.60);
\draw[dashed](s5.-150)--(s6.-30);
\draw[dashed](s6.-150)--(s7.-30);
\draw[dashed](s7.-150)--(s8.-30);
\draw[dashed](s8.-150)--(s9.-30);
\draw[dashed](s9.-165)to[out=-165, in=165,looseness=8] (s9.165);

\node[input] (ss) at (0.3, 0.8) {Start Area};
\node[system,fit=(ss) (s1) (s2) (s3) (s4)] {};
\node[null] (i5) at (5.5,0.1) {};
\node[null] (i6) at (4.2,0.1) {};
\node[null] (i7) at (2.9,0.1) {};
\node[null] (i8) at (1.6,0.1) {};
\draw[dashed](s5.north)--(i5);
\draw[dashed](s6.north)--(i6);
\draw[dashed](s7.north)--(i7);
\draw[dashed](s8.north)--(i8);
\end{tikzpicture}}
\vskip -0.3cm
\caption{Collision Problem. Solid lines depict target policy $\pi($forward$|\cdot)=1$. Dashed lines indicate the behavior policy $\mu($forward$|x\in\mathcal{S}_{\text{start}}\cup S_9)=1, \mu($forward$|x\in\mathcal{S}_{\text{next}})=0.5$, where  $\mathcal{S}_{\text{start}}=\{S_1, S_2, S_3, S_4\}$ and $\mathcal{S}_{\text{next}}=\{S_5, S_6, S_7, S_8\}$. On a {\em retreat} action, the agent goes back to a random state in  $\mathcal{S}_{\text{start}}$. }
\label{fig:collision_mdp}
\end{figure}
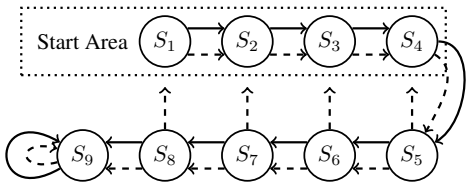

We empirically analyze the properties of these new emphatic algorithms to observe how qualitative properties such as convergence, learning speed and variance manifest in practice. We examine these in the context of two small scale diagnostic off-policy policy evaluation problems: (1) a two-state MDP, shown in Figure~\ref{fig:2state_mdp}, commonly used to highlight the instability of off-policy TD with function approximation, and (2) the Collision Problem, shown in Figure ~\ref{fig:collision_mdp}, used in prior work to highlight the advantages of ETD compared with gradient TD methods such as TDC~\cite{ghiassian2018online}.  In both cases we use linear function approximation, with a feature representation that includes significant generalization. In the appendix we also report experiments with Baird's counterexample ~\cite{baird1995}.

The results that follow are produced by extensive sweeps over the key hyper-parameters: we tested all combinations of the learning rate $\alpha \in\{2^i~|~ i\in -14,-13,...,-2\}$ and bootstrap length $n \in \{1,2...,5\}$; we selected the best hyper-parameters for each method by computing the RMSE over all time steps, and averaging results over many independent replications of the experiment---50 runs for the two-state MDP and 200 for the Collision problem. We then report both learning curves---plotting the RMSE over time for the best hyper-parameter setting from the sweep, and parameter studies---showing the average total RMSE for each algorithm, and for each combination of $n$, and $\alpha$. 

\begin{figure}[t]
\centering
\includegraphics[width=0.8\columnwidth]{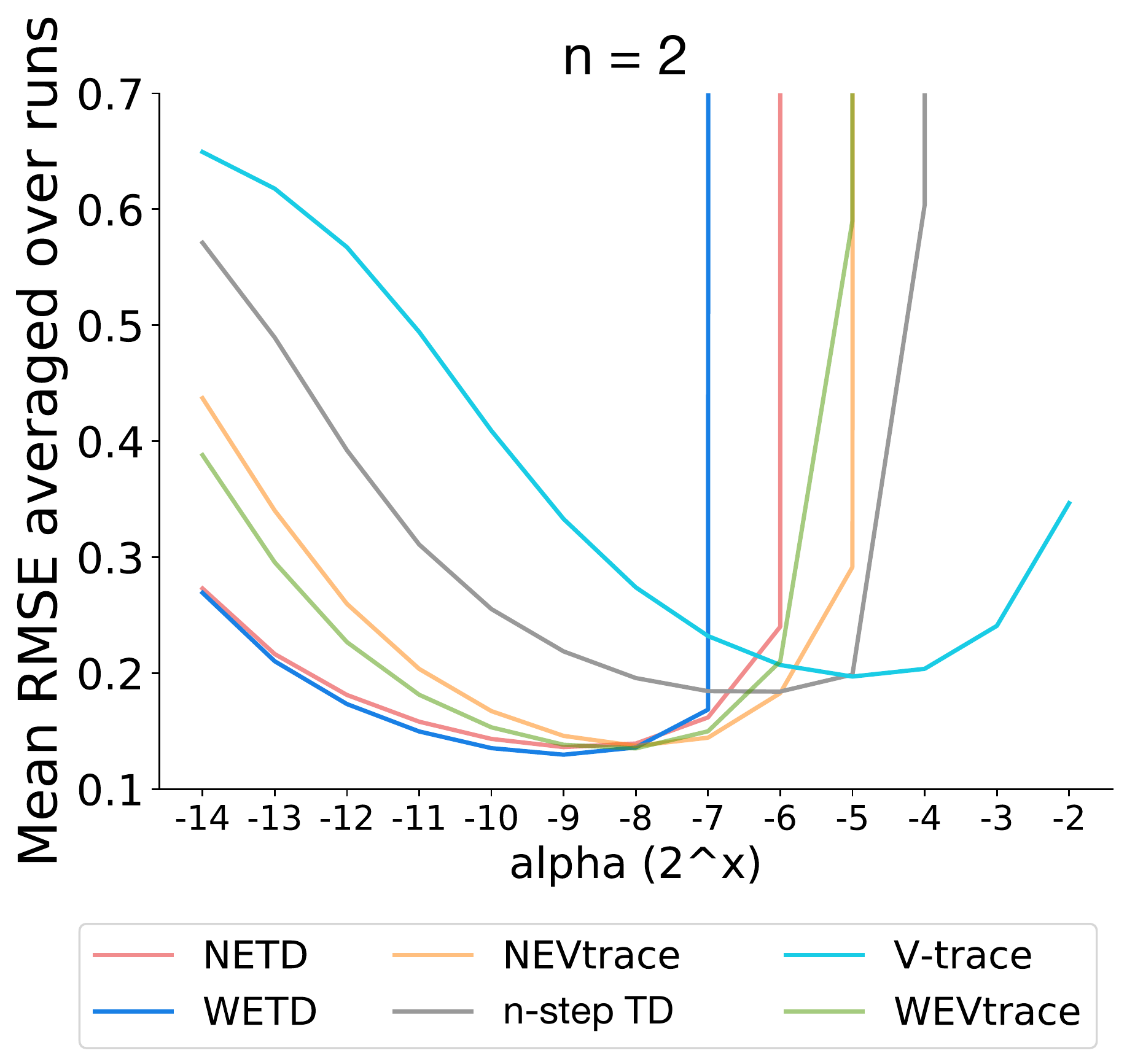}
\caption{Hyper-parameter sensitivity comparison on the Collision problem. Each data point in the plot shows the mean RMSE averaged over 200 runs for different values of learning rate $\alpha$. The emphatic algorithms achieve best performance in this task, and do so with smaller step-sizes than V-trace and n-step TD; consistent with previous results on this task \cite{ghiassian2018online}}
\vspace{1.2cm}
\label{fig:collision_u_curve}
\end{figure}

\subsection{Two-state MDP}

\begin{figure}[t]
\centering
\subfigure[$n$-step TD target]{
\includegraphics[width=\columnwidth]{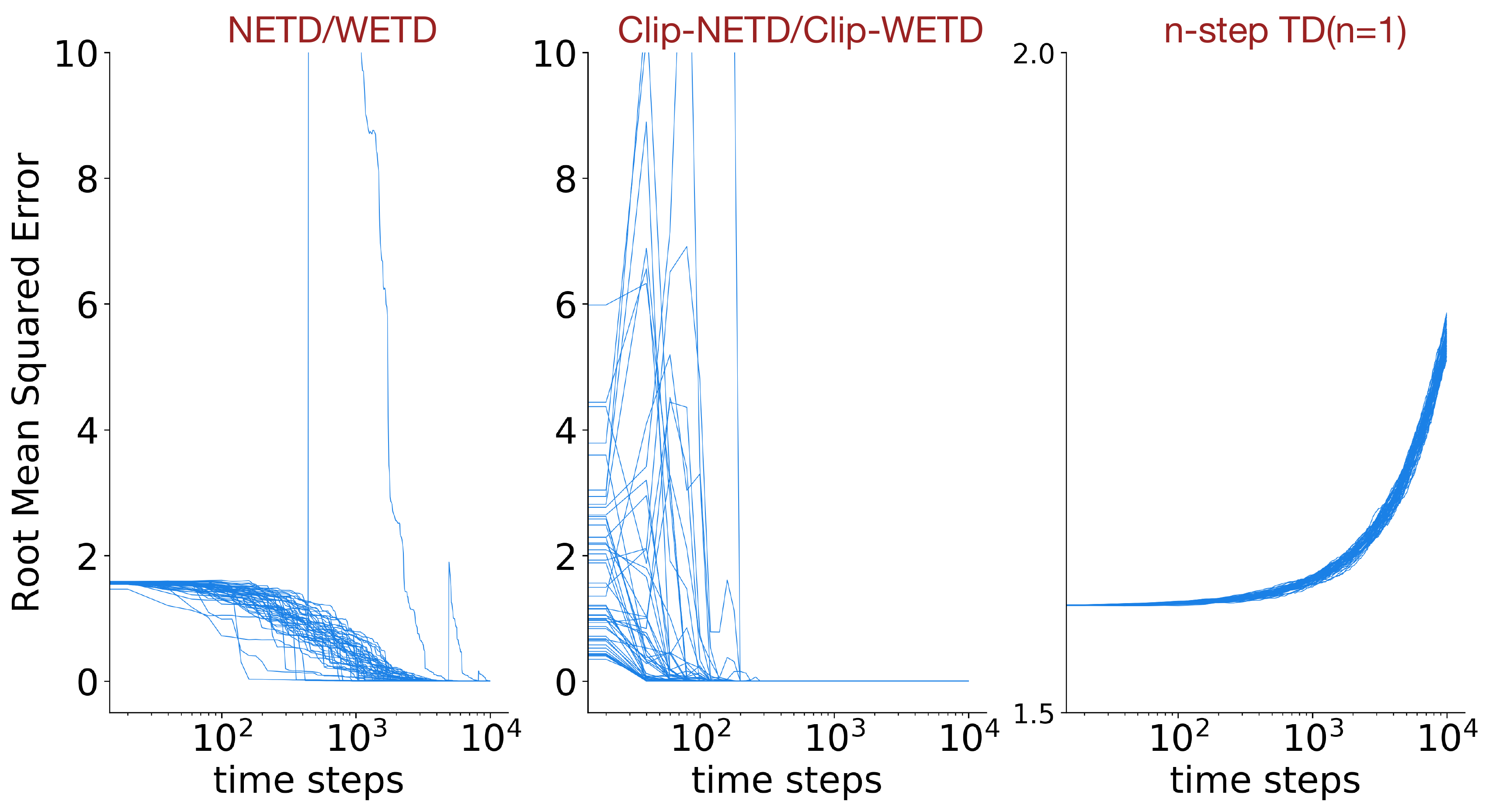}
\label{fig:2state_3_algs}
}
\subfigure[V-trace target]{
\includegraphics[width=\columnwidth]{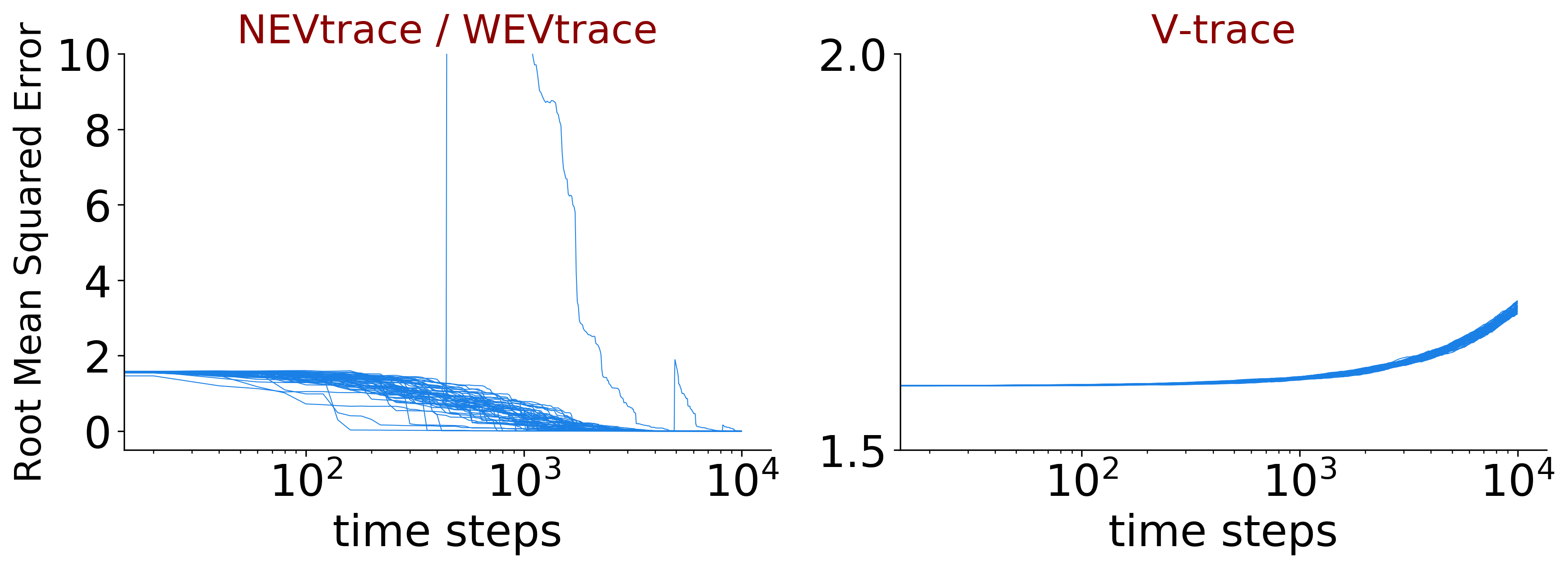}
\label{fig:2state_2_algs}
}
\vspace{-0.2cm}
\caption{Stability and learning speed in the two-state problem (with $n=1$ and 50 runs). Each plot reports the Root Mean Squared Error for all random seeds, using the best learning rate $\alpha$ for each method. 
\subref{fig:2state_3_algs}
TD(0) slowly diverged; NETD learned slowly and exhibited significant instability, even late in learning. Clip-NETD learned quickly and exhibits no instability beyond a few initial fluctuations. 
\subref{fig:2state_2_algs}
All runs of V-trace diverged, regardless of $\alpha$; NEVtrace did converge, but exhibited instability in some runs. Note the log scale on x-axis.}
\vspace{-0.5cm}
\end{figure}

\begin{figure}[t]
\centering
\includegraphics[width=\columnwidth]{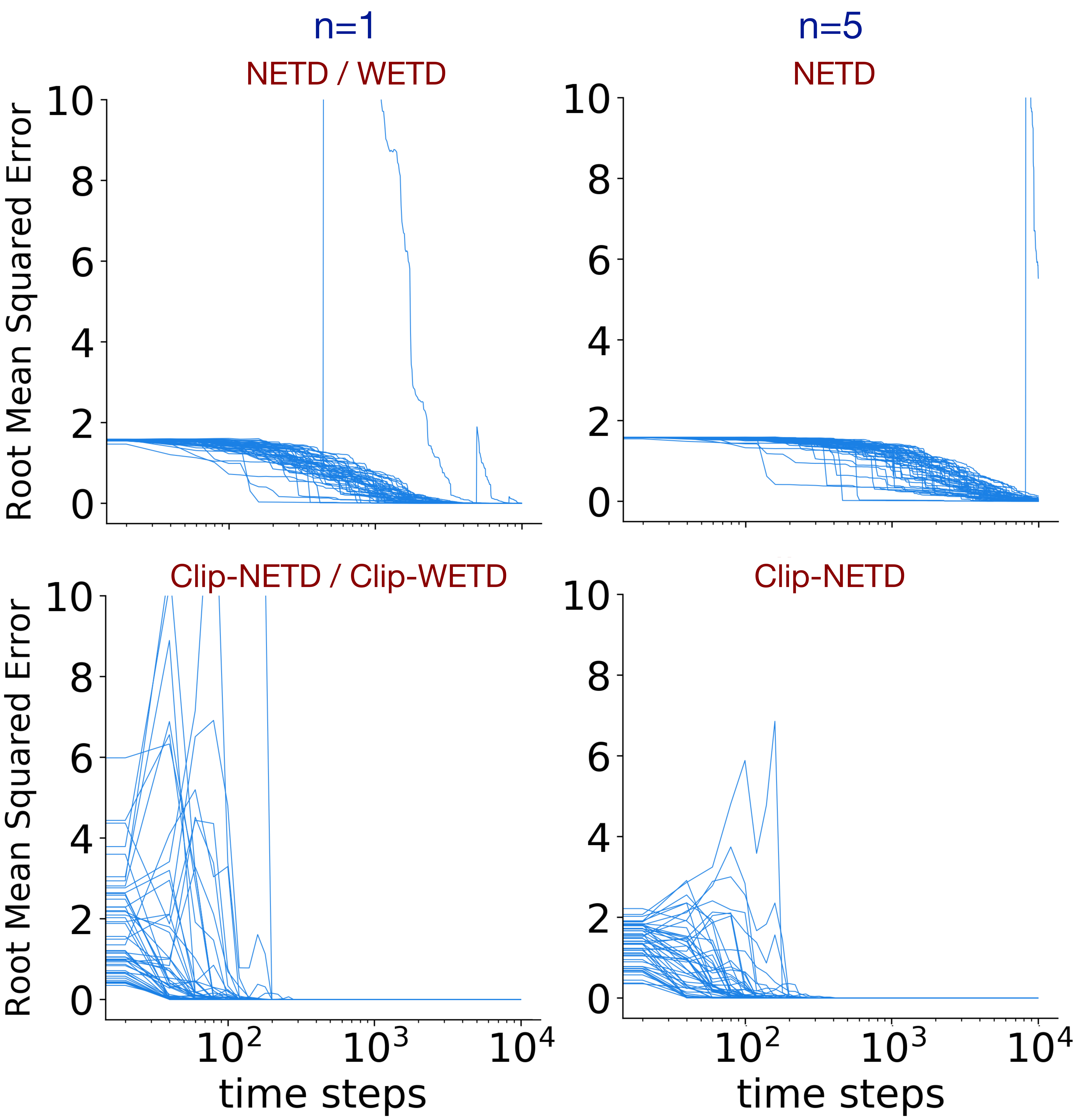}
\caption{Variance reduction due to clipping the IS weights in emphatic trace computation. Left column: $n=1$, right column: $n=5$, both with 50 runs. Clip-NETD (Eq.~\ref{eq:clip_netd}) and Clip-WETD (Eq.~\ref{eq:clip_wetd}) exhibit faster learning with no instability issues compared with the non-clipped version NETD / WETD. As $n$ gets larger, NETD becomes more stable but still exhibits spikes late in learning, whereas Clip-NETD remains stable and improves in initial learning. Note the log scale on x-axis.
}
\vspace{0.3cm}
\label{fig:2state_clip}
\end{figure}

This two-state MDP, illustrated in Fig~\ref{fig:2state_mdp}, is classical off-policy policy evaluation problem. Training data is generated by a random walk behavior policy. The task is to evaluate the target policy that always goes right, and ends up stuck in the second state forever. The values for the two states are approximated as $\theta$ and $2\theta$ with state features being scalars 1 and 2. The discount $\gamma$ is 0.9. Rewards are zero everywhere.

First we examine the convergence properties of various methods using the 1-step TD, a.k.a. TD(0) learning target. Empirically and theoretically, smaller bootstrap lengths $n$ are more likely to induce divergence in learning (further analysis and empirical evidence is in the appendix). Notice that the emphatic traces NETD and WETD are equivalent when applied to TD(0) (compare the formula of NETD in Eq.~\ref{eq:netd} with that of WETD in Eq.~\ref{eq:wetd} when $n=1$).

Figure~\ref{fig:2state_3_algs} presents Root Mean Squared Error (RMSE) over training time for NETD (WETD) and Clip-NETD (Clip-WETD), with the baseline learner TD(0) in the bottom panel. TD(0) diverged faster as training goes on. NETD converged slowly and exhibited significant instability: occasionally runs diverged even late into training. Clip-NETD by comparison learned quickly and exhibited low variance with no instability after some initial fluctuations. Since this small diagnostic MDP was designed to induce instability in learning, the initial fluctuations are expected due to the adversarial initialization of the value function parameters. Note we plot the error in the value estimates, but the algorithms are not directly optimizing value error. This is similar to previous results of \citet{sutton2018, sutton2016emphatic}, and like previous works, we plot all individual runs in order to highlight any instability in training. Overall, clipping IS weights proved to be an effective way of variance reduction at any bootstrap length $n$. Figure~\ref{fig:2state_clip} shows an example of the variance reduction effect for both $n=1$ (left column) and $n=5$ (right column). In both cases, Clip-NETD learned faster with no instability issues. 

Figure~\ref{fig:2state_2_algs} compares NEVtrace (or WEVtrace) to the corresponding V-trace baseline at $n=1$, that is equivalently, TD(0) with clipped IS weights. While V-trace diverged, NEVtrace converged slowly, although with instability issues and occasional spikes in error late in training in a subset of the runs. 

The full set of experiment results on the two-state MDP for all algorithms listed in Table~\ref{tab:methods} are included in the appendix, as well as results on the Baird's MDP with $n=1$ and $n=5$, which yield similar conclusions.

\subsection{Collision Problem}
In the Collision Problem (illustrated in Fig.~\ref{fig:collision_mdp}), states are aligned in a hallway and the agent can move forward or retreat. Episodes begin in one of the first four states, and terminates after $100$ time steps. The reward is zero on every transition, except on the transition into the last state $S_9$. The behavior policy always moves forward in the starting states, and outside of this area, either moves forward or retreats to the starting states with equal probability, except in the last trapping state $S_9$. The target policy moves forward in every state. We examine the emphatic algorithms in this environment through a hyper-parameter sensitivity study on the learning rate $\alpha$ and bootstrap length $n$. Figure~\ref{fig:collision_u_curve} presents the mean RMSE averaged across 200 runs of the emphatic algorithms and baselines $n$-step TD and V-trace at $n=2$, varying the learning rate $\alpha$. Emphatic algorithms achieved best performance---with best performance using smaller learning rates compared to the two baselines---consistent with previous results on ETD($\lambda$) in this task \cite{ghiassian2018online}. Additional training curves and hyper-parameter study plots for $n=1,2,3,5$ are in the appendix, supporting the same conclusion.

\section{Experiments at Scale}
\label{sec:atari_exp}

Our ultimate goal is to design emphatic algorithms that improve off-policy learning at scale, especially on actor-critic agents. Thus we evaluated the emphatic algorithms on Atari games from the Arcade Learning Environment \cite{bellemare2013}, a widely used deep RL benchmark. 

\paragraph{Data} We use the raw pixel observations in RGB as they are provided by the environment, without down sampling or gray scaling them. We also use an action repeat of $4$, with max pooling over the last two frames and the life termination signal. This setup is similar to IMPALA \cite{espeholt2018} with the only difference being using the raw frames instead of down and gray scaled ones. In order to compare with closely related previous works, we adopted the conventional 200M frames training regime using online updates without experience replay. 

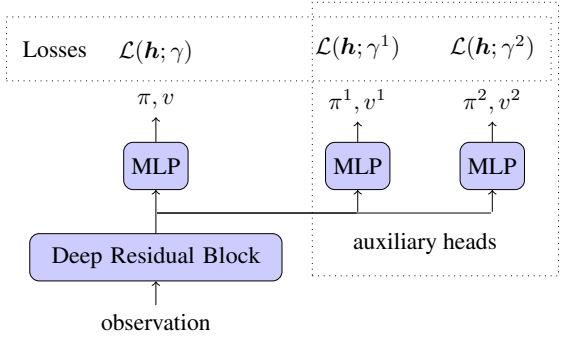
\begin{figure}[t]
\centering
\resizebox{0.9\columnwidth}{!}{
\begin{tikzpicture}
\node (resblock) [mlp] {Deep Residual Block};
\path (resblock.south)+(0, -0.6) node (obs) [inp] {observation};
\path (resblock.north)+(0, 1) node (mlp0) [op] {MLP};
\path (resblock.north)+(3, 1) node (mlp1) [op] {MLP};
\path (resblock.north)+(5, 1) node (mlp2) [op] {MLP};
\path (mlp0.north)+(0, 0.6) node (out0) [inp] {$\pi, v$};
\path (mlp1.north)+(0, 0.6) node (out1) [inp] {$\pi^1, v^1$};
\path (mlp2.north)+(0, 0.6) node (out2) [inp] {$\pi^2, v^2$};
\path (out0.north)+(0, 0.5) node (loss0) [inp] {$\mathcal{L}(\vh;\gamma)$};
\path (out1.north)+(0, 0.5) node (loss1) [inp] {$\mathcal{L}(\vh;\gamma^1)$};
\path (out2.north)+(0, 0.5) node (loss2) [inp] {$\mathcal{L}(\vh;\gamma^2)$};
\node[input] (l) at (-1.5, 3.1) {Losses};
\node[system,fit=(loss0) (loss1) (loss2) (l)] {};
\path (resblock.north)+(0, 0.3) node (null0) [null] {};
\path (null0)+(3, 0) node (null1) [null] {};
\path (null0)+(5, 0) node (null2) [null] {};
\path (null2)+(0, 3.) node (null_p1) [null] {};
\path (null_p1)+(0.85, 0) node (null_p2) [null] {};
\node[input] (aux) at (4, 0.2) {auxiliary heads};
\node[system,fit=(mlp1) (mlp2) (out1) (out2) (aux) (null1) (null2) (null_p1) (null_p2)] {};

\path [draw, ->] (obs.north) -- node [above] {} (resblock.south);
\path [draw, ->] (resblock.north) -- node [above] {} (mlp0.south);
\path [draw, -] (null0) -- node [above] {} (null1);
\path [draw, -] (null0) -- node [above] {} (null2);
\path [draw, ->] (null1) -- node [above] {} (mlp1.south);
\path [draw, ->] (null2) -- node [above] {} (mlp2.south);
\path [draw, ->] (mlp0.north) -- node [above] {} (out0.south);
\path [draw, ->] (mlp1.north) -- node [above] {} (out1.south);
\path [draw, ->] (mlp2.north) -- node [above] {} (out2.south);

\end{tikzpicture}}
\caption{Block diagram of Surreal, with one main head and two auxiliary heads. IMPALA loss on each head uses different discounts $\gamma, \gamma^1, \gamma^2$. Let $\vh$ denote the neural network model. The behavior policy is fixed to be $\pi$.}
\label{fig:surreal}
\end{figure}

\begin{algorithm}[t]
\textbf{Input: Bootstrapping length $n$, discounts $\gamma, \gamma^1, \gamma^2$, number of actors $M$.}\\
Initialize Surreal neural network function $\vh_0$, \\
Output initial policy for the main head $\pi_0$ from $\vh_0$. \\
\For{actor $m\in[1, \ldots, M]$}{
  Sample trajectory $(S^m_i, A^m_i, R^m_{i+1})_{i=0}^{n-1} \sim \pi_0$.\\
  \For{auxiliary head $u=1, 2$}{
    Initialize $F^{(n),m,u}_0, \ldots, F^{(n),m,u}_{n-1}=1$.\\
    \For{$i=0, \ldots, n-1$}{
        Set $\bar{\rho}^{m, u}_i = \min(1, \frac{\pi^u_0(A^m_i|S^m_i)}{\pi_0(A^m_i|S^m_i)})$.
    }
  }
}
\For{timestep $t\in[0, \ldots, T]$}{
  Output main head policy $\pi_t$ from neural network $\vh_t$.
  \For{actor $m\in[1,\ldots,M]$}{
    Sample $S^m_{t+n}, A^m_{t+n}, R^m_{t+n+1} \sim \pi_t$.\\
    Compute the main head IMPALA loss $\mathcal{L}^m_t(\vh;\gamma)$.\\
    \For{auxiliary head $u=1, 2$}{
       Set $\bar{\rho}^{m, u}_{t+n} = \min(1, \frac{\pi^u_t(A^m_{t+n}|S^m_{t+n})}{\pi_t(A^m_{t+n}|S^m_{t+n})})$.\\
       \textbf{if} $t\geq n$ \textbf{then}\\
       $F^{(n),m,u}_t = \prod_{i=1}^{n}(\gamma^u_{t-i+1} \bar{\rho}^{m, u}_{t-i}) F^{(n),m,u}_{t-n} + 1$,\\
       \textbf{end}\\
       Weight the sum of IMPALA value and policy losses, plus IMPALA entropy loss:\\
       $\mathcal{E}^{m,u}_t =F^{(n),m,u}_t \mathcal{L}^{m,u}_t(\vh;\gamma^u) + \mathcal{H}^{m,u}_t$.\\
    }
  }
  Update neural network $\vh_{t+1}$ using an average loss: \\
  $\mathcal{L}_t = \frac{1}{3M} \sum_m (\mathcal{L}^m_t + \mathcal{E}^{m,1}_t+\mathcal{E}^{m,2}_t).$\\
}
\textbf{Return:} $h_T$.
\caption{NETD-ACE Surreal.}
\label{alg:netd_surreal}
\end{algorithm}

\paragraph{Agent} StacX \cite{zahavy2020} and UNREAL \cite{jaderberg2016} are both IMPALA-based agents that learn auxiliary tasks from experience generated by the main policy, in order to improve the shared representation. Inspired by their results, we investigated whether emphatic algorithms can help learn the auxiliary tasks better since they are learned off-policy, and in turn improve the agent performance. In particular, we used an IMPALA-based agent with two auxiliary heads, each head learning a different target policy for its own discount $\gamma, \gamma^1, \gamma^2$ (see Fig.~\ref{fig:surreal} and the appendix for details on its network structures and hyper-parameters). We call this agent {\em Surreal} as it fantasizes (learns off-policy) about two additional policies $\pi^1$, $\pi^2$ that discount the future rewards differently, without ever executing actions from them. We apply emphatic traces to the IMPALA learning updates on the two auxiliary heads. In order to reduce variance, we always clip the IS weights at 1 both in computing emphatic traces and in the V-trace target. In a distributed system, we keep track of an emphatic trace for each actor's trajectories and aggregate the updates in a batch average at every time step. Algo.~\ref{alg:netd_surreal} outlines the pseudo-code for NETD-ACE Surreal as an example. For the implementation of Surreal, we used Jax libraries \cite{rlax2020, haiku2020, optax2020} on a TPU Pod infrastructure called Sebulba \citep{sebulba2021}.

\paragraph{Evaluation} We compute the median human normalized scores across 57 games, averaged over seeds and an {\em evaluation phase} without learning. To compare any two agents, we view their scores on 57 games as 57 independent pairs of samples, similar to how one would test significance of a medical treatment on a population of different people, rather than testing same treatment on the same person multiple times. The $p$-value is the probability of the null hypothesis that the algorithm performs equally using the {\em sign test} \citep{Arbuthnot1712}. Results might be thought of as statistically significant when $p<0.05$.

\paragraph{Baselines} Prior to applying emphatic traces, we found the best hyper-parameters for Surreal in the mixed and the fixed update schemes separately as our baselines. In the mixed update scheme, $n=40, \alpha=6\cdot10^{-4}, \text{max gradient norm}=0.3$ yielded the best results. In the fixed update scheme, the best hyper-parameters for Surreal were $n=10, \alpha=2\cdot10^{-4}, \text{max gradient norm}=1$. 

\paragraph{Emphatic Results} Since the emphatic traces are derived using the steady state distributions following fixed policies, we expect that they would impact the results more towards the end of learning, as the agent stabilizes its learned policy with learning rate decay. Empirically we observed the differences between algorithms start to show around $130$M frames or $65\%$ of learning frames.

\begin{figure}[t]
\centering
\includegraphics[width=\columnwidth]{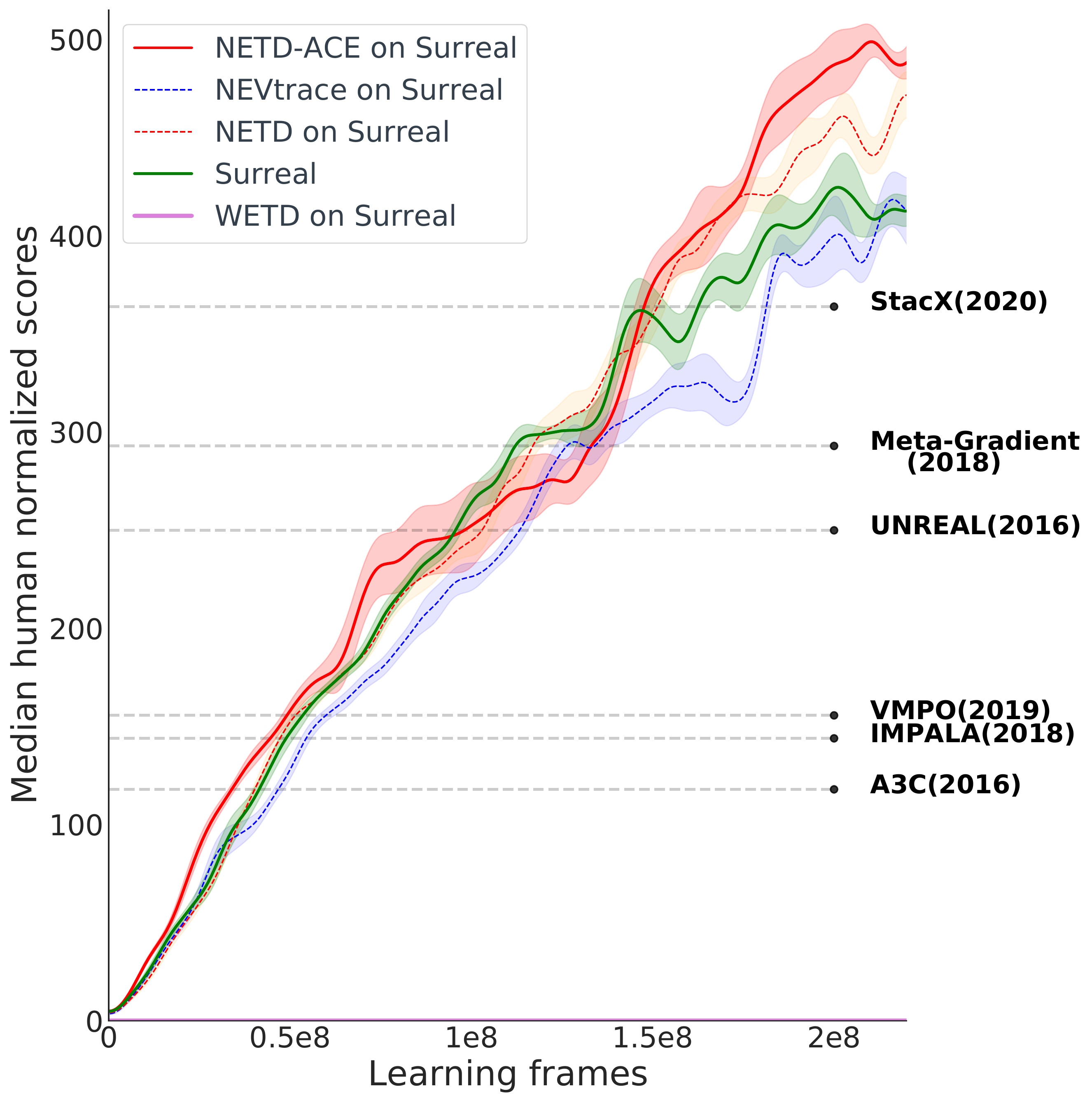}
\caption{Learning curves of baseline Surreal and the emphatic traces (NETD, NETD-ACE, NEVtrace, WETD) applied to Surreal, in the fixed update scheme with $n=10$, IS weights clipped to 1. Median human normalized scores are averaged across 3 random seeds with shaded areas denoting standard derivations.}
\vspace{0.5cm}
\label{fig:surreal_fixed_n_3}
\end{figure}

In the mixed update scheme, we tested emphatic trace family WETD and its variants WETD-ACE, WEVtrace applied to the Surreal baseline. In order to reduce instability, we experimented with several variance reduction techniques, including: 1) ETD($\lambda, \beta$), 2) interpolation, and 3) clipping. First, previous work on ETD($\lambda, \beta$)\cite{hallak2016} suggested using a hyper-parameter $\beta$ to replace the discount variable $\gamma_t$ in the follow-on trace, which we applied to WETD.
Second, we introduced a constant hyper-parameter $\eta\in(0,1)$ such that $M^w_t = 1 - \eta(1 - \lambda_t) + \eta(1 - \lambda_t) F_t$, allowing us to modify the interpolation between a potentially large $F_t$ and 1 to restrain the blow-up.
Third, we clipped the values of $\rho_{t-1}$ and/or directly clipped the values of $F_t$. However, despite of these efforts, WETD still diverged with exploding gradients (see WETD in Fig.~\ref{fig:surreal_fixed_n_3}).  

Next, we evaluated emphatic trace family NETD (dashed orange) and its variants NETD-ACE (solid red), NEVtrace (dashed blue), applied to Surreal, along with the baseline Surreal agent (solid green) using a fixed update scheme. The best Surreal baseline from our sweeps already surpassed the StacX scores \cite{zahavy2020}. The results are averaged over 3 random seeds, and Fig.~\ref{fig:surreal_fixed_n_3} depicts the learning curves and Table~\ref{tab:atari_fixed} summarizes all performance statistics. In particular, the best performing emphatic actor-critic agent NETD-ACE improved the median human normalized score from the baseline performance of $403\%$ to $497\%$, the highest score for an RL agent without experience replay in the 200M frames data regime. It improved performance on $100$ out of $57 \times 3$ Atari games compared to the baseline, with a p-value of $0.016$, achieving $95\%$ statistical significance.

\begin{table}[t]
\begin{center}
\scalebox{0.8}{%
\begin{tabular}{l|cccc}
Statistics & NETD-ACE & NETD & NEVtrace & Surreal\\
\toprule
Median                    &\textbf{497.21}  &427.69   &317.50   &403.47\\
Mean                      &\textbf{1793.47} &1507.85  &1502.84  &1565.42\\
$40$th percentile          &\textbf{300.26}  &268.43   &186.95   &258.35\\
$30$th percentile           &162.91  &\textbf{169.42}   &118.27   &163.90\\
$20$th percentile           &\textbf{74.47}   &70.74    &28.28    &65.60\\
$10$th percentile           &4.88    &4.3      &4.91     &4.25\\
$\#$ games \textgreater human  &\textbf{45}/57   &\textbf{45}/57    &43/57    &43/57\\
\bottomrule
\end{tabular}}
\end{center}
\caption{Performance statistics for baseline Surreal and emphatic traces applied to Surreal in the fixed update scheme with $n=10$, on 57 Atari games. Scores are human normalized, averaged across 3 random seeds and across the evaluation phase.}
\label{tab:atari_fixed}
\end{table}

\section{Discussion}

New emphatic algorithm families of WETD and NETD variants showed nice qualitative properties on off-policy diagnostic MDPs. For both families, clipping IS weights in computing emphatic traces turns out to be an effective way to reduce variance, so we applied this learning at scale. On Atari, we proposed a baseline agent Surreal that achieved a strong median human normalized score $403\%$, and is suitable for testing off-policy learning on auxiliary controls. The WETD family were unstable at scale, whereas the NETD family performed well, particularly the emphatic actor-critic agent NETD-ACE. For future work, we would like to investigate applying emphatic traces to a variety of off-policy learning targets and settings at scale. 

\newpage
\bibliography{emphatic}
\bibliographystyle{icml2021}

\appendix
\newpage
\onecolumn

\icmltitle{Supplementary Material}
\section{Stability in Off-policy $n$-step TD Learning}
In the context of off-policy $n$-step TD learning, we discuss the possible occurrence of unstable learning in general as well as characterize a safety region in the policy space that guarantees stable learning.

\subsection{Unstable Learning}
\label{app:unstable_tdn}
Following the same notations as in the main paper, let state value functions be approximated by a linear function approximation $\vtheta_t^T \vphi(S_t)$ where $\vphi$ are feature maps. For the $n$-step TD target, the value function update on $\vtheta_t$ is
\begin{align}
    \vtheta_{t+1} &= \vtheta_t + \alpha \sum_{i=t}^{t+n-1} \prod_{j=t}^{i-1}(\gamma_{j+1}\rho_j) \rho_i (R_{i+1} + \gamma_{i+1} \vtheta_t^T \vphi(S_{i+1}) - \vtheta_t^T \vphi(S_i) ) \vphi(S_t)\\
    &=\vtheta_t + \alpha\left(\underbrace{\sum_{i=t}^{t+n-1} \prod_{j=t}^{i-1}(\gamma_{j+1}\rho_j)\rho_i R_{i+1} \vphi(S_t)}_{\vb_t} - \underbrace{\vphi(S_t) \sum_{i=t}^{t+n-1} \prod_{j=t}^{i-1}(\gamma_{j+1}\rho_j)\rho_i \left[\vphi(S_i) - \gamma_{i+1} \vphi(S_{i+1})\right]^T}_{\mA_t} \vtheta_t  \right)\\
    &=\vtheta_t + \alpha(\vb_t - \mA_t\vtheta_t) = (\mI-\alpha \mA_t)\vtheta_t + \alpha \vb_t.
\end{align}

To achieve stability as defined in \citet{sutton2016emphatic}, we need $\vb_t$ and $\mA_t$ to converge to unique fixed points $\vb$ and $\mA$, and we need the steady state updates to be stable regardless of the initial parameters $\vtheta_0$, equivalently requiring $\mA$ to be a positive-definite matrix. 
\begin{align}
    \mA = \lim_{t\rightarrow\infty} \E[\mA_t] &= \lim_{t\rightarrow\infty} \E_{\mu}\vphi(S_t) \sum_{i=t}^{t+n-1} \prod_{j=t}^{i-1}(\gamma_{j+1}\rho_j)\rho_i \left[\vphi(S_i) - \gamma_{i+1} \vphi(S_{i+1})\right]^T \\
    &= \sum_s d_{\mu} (s) \E_{\mu} \left[\vphi(S_t) \sum_{i=t}^{t+n-1} \prod_{j=t}^{i-1}(\gamma_{j+1}\rho_j)\rho_i \left[\vphi(S_i) - \gamma_{i+1} \vphi(S_{i+1})\right]^T \middle\vert S_t=s\right]\\
    &= \sum_{i=t}^{t+n-1} \sum_s d_{\mu} (s) \E_{\mu} \left[\vphi(S_t) \prod_{j=t}^{i-1}(\gamma_{j+1}\rho_j)\rho_i \left[\vphi(S_i) - \gamma_{i+1} \vphi(S_{i+1})\right]^T \middle\vert S_t=s\right]\label{eq:dev1}\\
    &= \mPhi^T \mD_\mu\left[(\mI - \mP_\pi\mGamma) + (\mP_\pi\mGamma - \mP^2_\pi\mGamma^2) + \ldots + (\mP^{n-1}_\pi \mGamma^{n-1} - \mP^n_\pi \mGamma^n)\right] \mPhi\label{eq:dev2}\\
    &= \mPhi^T \mD_\mu\left[\mI- \mP^n_\pi\mGamma^n\right] \mPhi, \label{eq:A_lim}
\end{align}
where $\mGamma$ is a diagonal matrix with diagnal entries $\mGamma_{t,t} = \gamma_t \doteq \gamma(S_t)$. To see the derivation from Eq.~\ref{eq:dev1} to Eq.~\ref{eq:dev2}, take one term in the sum indexed by $i$. We have

\begin{align}
    & \sum_s d_{\mu} (s) \E_{\mu} \left[\vphi(S_t) \prod_{j=t}^{i-1}(\gamma_{j+1}\rho_j)\rho_i \left[\vphi(S_i) - \gamma_{i+1} \vphi(S_{i+1})\right]^T \middle\vert S_t=s\right]\\
    = &\sum_s d_{\mu} (s) \vphi(s) \sum_{A_t, S_{t+1} \ldots, A_i, S_{i+1}} \prod_{k=t}^{i}\mu(A_k|S_k) p(S_{k+1}|S_k, A_k) \frac{\pi(A_k|S_k)}{\mu(A_k|S_k)} \prod_{j=t}^{i-1}\gamma_{j+1}\left[\vphi(S_i) - \gamma_{i+1} \vphi(S_{i+1})\right]^T\\
    = &\sum_s d_{\mu} (s) \vphi(s) \sum_{A_t, S_{t+1} \ldots, A_i, S_{i+1}} \prod_{k=t}^{i} p(S_{k+1}|S_k, A_k) \pi(A_k|S_k) \prod_{j=t}^{i-1}\gamma_{j+1}\left[\vphi(S_i) - \gamma_{i+1} \vphi(S_{i+1})\right]^T\\
    = &\sum_s d_{\mu} (s) \vphi(s) \E_{\pi} \left[\prod_{j=t}^{i-1}\gamma_{j+1}\left[\vphi(S_i) - \gamma_{i+1} \vphi(S_{i+1})\right]^T\middle\vert S_t=s\right]\\
    = &\sum_s d_{\mu} (s) \vphi(s)
    \sum_{S_{t+1}} \gamma_{t+1} [\mP_{\pi}]_{S_t S_{t+1}}\cdots\sum_{S_i} \gamma_i [\mP_{\pi}]_{S_{i-1} S_i} \left[\vphi(S_i) - \sum_{S_{i+1}} \gamma_{i+1} [\mP_{\pi}]_{S_i S_{i+1}} \vphi(S_{i+1}) \right]^T\\
    = &\mPhi^T \mD_{\mu} (\mP_{\pi}^{i-t}\mGamma^{i-t} - \mP_{\pi}^{i-t+1}\mGamma^{i-t+1} ) \mPhi 
\end{align}

Back to Eq.~\ref{eq:A_lim}, the resulting matrix $\mA=\mPhi^T \mD_\mu\left[\mI- \mP^n_\pi\mGamma^n\right] \mPhi$ is not necessarily positive definite since the key matrix $\mD_\mu\left[\mI- \mP^n_\pi\mGamma^n\right]$ can be non-positive definite. For example, in the two-state MDP when $n=2$, let the discount $\gamma=0.99$. We know that the steady state distribution following the behavior policy is equal probability of being in either state, and the target policy always goes right, i.e. 

\begin{align}
\mD_{\mu} = 
\begin{bmatrix}
0.5 & 0\\
0 & 0.5
\end{bmatrix},
\mP_{\pi} = 
\begin{bmatrix}
0 & 1\\
0 & 1
\end{bmatrix}.
\end{align}

Hence the key matrix is 
\begin{align}
    \mD_\mu\left[\mI-\gamma^2 \mP^2_\pi\right] = 
    \begin{bmatrix}
        0.5 & -0.99^2/2\\
        0 & (1-0.99^2)/2
    \end{bmatrix}.
\end{align}
It is not a positive definite matrix through checking multiplication on both sides by setting $\mPhi=(1,2)^T$.

\subsection{Safety Guarantee}
\label{app:safety_region}

Recall that the key matrix of the TD(0) algorithm is given by $\mD_\mu\left[\mI -\mP_{\pi}\mGamma\right]$. We now briefly summarize a few facts about the key matrix from \cite{sutton2016emphatic}. First, the diagonal entries of the key matrix are positive and the off-diagonal entries are negative, so in order to show its positive definiteness it is enough to show that each row
sum plus the corresponding column sum is positive. The row sums are all positive because $\mP_\pi$ is a stochastic matrix and values in $\mGamma$ are smaller than 1. Thus it only remains to show that the column sums are non-negative. The problem is that this is not true for a general $\mD_\pi, \mD_\mu$ as was shown in \cite{sutton2016emphatic}. We further showed that this is not true for general $n$-step TD learning (App.~\ref{app:unstable_tdn}). Yet, we show next that for a distribution $\mD_\mu$ that is close enough to $\mD_\pi$ the key matrix is still positive definite. Intuitively, the implication of this result is that doing off-policy learning with $\mD_\pi \sim \mD_\mu$ is stable. To show that, we need to show that the column sums of the key matrix are all positive. We begin by lower bounding them as follows:

\begin{align}
\mathbf{1}^T \mD_\mu\left[\mI- \mP_{\pi}\mGamma\right] & = \vd_\mu^T \left[\mI- \mP_{\pi}\mGamma\right] \\
& = (\vd_\pi + \vd_\mu - \vd_\mu)^T \left[\mI- \mP_{\pi}\mGamma\right] \\
& = \vd_\pi^T\left[\mI- \mP_{\pi}\mGamma\right] + (\vd_\mu - \vd_\pi)^T \left[\mI- \mP_{\pi}\mGamma\right] \\
& \ge \vd_\pi^T(\mI-\mGamma) + (\vd_\mu - \vd_\pi)^T \left[\mI - \mP_{\pi}\mGamma\right] \label{eq:lower_bound_key}.
\end{align}
It remains to show that \cref{eq:lower_bound_key} has only positive coordinates. 
The i-th coordinate is given by $\vd_\pi^T[\mI-\mGamma]_i + (\vd_{\mu} - \vd_{\mu})^T \left[\mI- \mP_{\pi}\mGamma\right]_i,$ where the subscript $i$ denotes the i-th column of a matrix. By Holder inequality, we have that:
\begin{align}
\vd_\pi^T[\mI-\mGamma]_i + (\vd_\mu - \vd_\pi)^T \left[\mI- \mP_{\pi}\mGamma\right]_i & \ge \vd_\pi^T[\mI-\mGamma]_i - ||\vd_\mu - \vd_\pi||_{\infty} ||\left[\mI- \mP_\pi \mGamma\right]_i|| _1.
\end{align}
The first term in the last equation $\vd_\pi^T[\mI-\mGamma]_i$ is positive and does not depend on $\mu$. The second term has two contributions. The first one, $||\vd_\mu - \vd_\pi||_{\infty}$ depends on $\mu$ and $\pi$ but it can become as small as we want in the limit that $\mu\rightarrow\pi$. The second quantity, $||\left[\mI- \mP_\pi \mGamma\right]_i|| _1$ depends only on $\pi$. This implies that for a fixed $\pi$ there exists a $\mu$ that is close enough to it such that the key matrix is positive definite. 

Note that we can repeat this analysis for $n$-step TD. In this case we need to show that the key matrix $\mD_\mu\left[\mI- \mP^{n}_\pi \mGamma^{n}\right]$ is positive definite (see \cref{app:unstable_tdn} above for its derivation). All the derivation we did above for the TD(0) applies to the $n$-step TD scenario by replacing $\mP_\pi$ with $\mP^n_\pi$. The only step which we need to justify is $\vd_\pi^T\left[\mI-\mP_{\pi}^n\mGamma^n\right] = \vd_\pi^T(\mI-\mGamma),$ but to see this recall that $\vd_\pi^T \mP_{\pi} = \vd_\pi^T$, so $\vd_\pi^T \mP_{\pi}^n = \vd_\pi^T \mP_{\pi}\mP_{\pi}^{n-1} = \vd_\pi^T \mP_{\pi}^{n-1} = \ldots= \vd_\pi^T.$

\section{WETD Derivation}
\subsection{TD($\lambda_t$) as mixed $n$-step TD target}
\label{app:lambda_tdn}
Defining $\lambda_t$ as in Eq.~\ref{eq:tdn_lambda}, we write out the off-policy TD$(\lambda_t)$ learning target by adding the importance sampling correction to Eq. (12.10) in \cite{sutton2018} for an arbitrarily large integer $q$
\begin{align}
\tilde{G}_t \doteq V(S_t) + \hskip-0.1cm \sum_{i=t}^{t+q-1} (\prod_{j=t}^{i-1} \rho_j \gamma_{j+1}\lambda_{j}) \hskip0.1cm \lambda_i \rho_i \delta_i.
\end{align}
In the sum of weighted TD errors from time $t$ to time $(t+q-1)$, the weight $\prod_{j=t}^{i} \lambda_j$ is zero if any of the values $\lambda_j=0$, and the TD return bootstraps at the first encounter of $\lambda_j=0$. For simplicity, first consider when $t=0$, the $n$-step TD corresponds to the truncated return where all terms in the sum are zero for $i\geq n$, requiring $\lambda_n=0$. At $t=n$, we have that $\lambda_{2n}=0$ produces the $n$-step TD learning target. In general, this corresponds to $\lambda_j=0$ whenever $j$ is a multiple of $n$. To check that this is the mixed update for any $k$-th sample in the trajectory,
\begin{align}
\label{eq:mixed_lambda}
\tilde{G}_{t+k} \doteq V(S_{t+k}) + \hskip-0.1cm \sum_{i=t+k}^{t+q-1} (\prod_{j=t+k}^{i-1} \rho_j \gamma_{j+1}\lambda_j) \hskip0.1cm \lambda_i\rho_i\delta_i = V(S_{t+k}) + \hskip-0.1cm \sum_{i=t+k}^{t+n-1} (\prod_{j=t+k}^{i-1} \rho_j \gamma_{j+1}) \hskip0.1cm \rho_i\delta_i,
\end{align}
since $t+n$ is the smallest number bigger than $t$ such that $t+n$ is a multiple of $n$. Thus we recover the mixed $n$-step update where each sample $V(S_{t+k})$ in the trajectory is updated with $(n-k)$-step TD error. 

\subsection{TD($\lambda_t$) as mixed V-trace target}
\label{app:lambda_vtrace}
In Eq.~\ref{eq:mixed_lambda} above, if we set $\lambda_j = \bar{\rho}_j/\rho_j$ as defined in Eq.~\ref{eq:vtrace_lambda}, we recover the mixed V-trace target.
\begin{align}
\label{eq:mixed_lambda_vtrace}
\tilde{G}_{t+k} \doteq V(S_{t+k}) + \hskip-0.1cm \sum_{i=t+k}^{t+n-1} (\prod_{j=t+k}^{i-1} \rho_j \gamma_{j+1}\lambda_j) \hskip0.1cm \lambda_i\rho_i\delta_i = V(S_{t+k}) + \hskip-0.1cm \sum_{i=t+k}^{t+n-1} (\prod_{j=t+k}^{i-1} \bar{\rho}_j \gamma_{j+1}) \hskip0.1cm \bar{\rho}_i\delta_i.
\end{align}

\section{NETD Derivation}
\label{app:netd_derivation}
In order to simplify notations in the computation below, we denote the NETD trace as $F$ by omitting the superscript $(n)$. Given the possibly unstable asymptotic updates shown in App.~\ref{app:unstable_tdn}, we can modify the updates with $F_t$ to ensure that the new limit matrix $\mA$ is positive definite, i.e. to stabilize learning. The $F_t$-modified parameter update is
\begin{align}
    \vtheta_{t+1} &= \vtheta_t + \alpha F_t\sum_{i=t}^{t+n-1} \prod_{j=t}^{i-1}(\gamma_{j+1}\rho_j)\rho_i(R_{i+1} + \gamma_{i+1}\vtheta_t^T\vphi(S_{i+1})-\vtheta_t^T\vphi(S_i)) \vphi(S_t)\\
    &=\vtheta_t + \alpha\left(\underbrace{F_t\sum_{i=t}^{t+n-1} \prod_{j=t}^{i-1}(\gamma_{j+1}\rho_j)\rho_i R_{i+1} \vphi(S_t)}_{\vb_t} - \underbrace{F_t\vphi(S_t) \sum_{i=t}^{t+n-1} \prod_{j=t}^{i-1}(\gamma_{j+1}\rho_j)\rho_i \left[\vphi(S_i) - \gamma_{i+1} \vphi(S_{i+1})\right]^T}_{\mA_t} \vtheta_t  \right).
\end{align}

Then the $\mA$ matrix for the emphatically modified $n$-step TD update becomes
\begin{align}
    \mA &= \lim_{t\rightarrow\infty} \E[\mA_t] = \lim_{t\rightarrow\infty} \E_{\mu}F_t\vphi(S_t) \sum_{i=t}^{t+n-1} \prod_{j=t}^{i-1}(\gamma_{j+1}\rho_j)\rho_i \left[\vphi(S_i) - \gamma_{i+1} \vphi(S_{i+1})\right]^T \label{line:l1}\\
    &= \sum_s d_{\mu} (s) \lim_{t\rightarrow\infty} \E_{\mu} \left[F_t\vphi(S_t) \sum_{i=t}^{t+n-1} \prod_{j=t}^{i-1}(\gamma_{j+1}\rho_j)\rho_i\left[\vphi(S_i) - \gamma_{i+1} \vphi(S_{i+1})\right]^T \middle\vert S_t=s\right]\label{line:l2}\\
    &= \sum_s d_{\mu} (s) \lim_{t\rightarrow\infty}  \E_{\mu} \left[F_t| S_t=s\right] \E_{\mu} \left[\vphi(S_t)  \sum_{i=t}^{t+n-1} \prod_{j=t}^{i-1}(\gamma_{j+1}\rho_j)\rho_i \left[\vphi(S_i) - \gamma \vphi(S_{i+1})\right]^T \middle\vert S_t=s\right]\label{line:l3}.
\end{align}    
Eq.~\ref{line:l2} is obtained from Eq.~\ref{line:l1} by using the linearity of expectation over the learning updates on different states weighted by their steady state visit frequencies. Eq.~\ref{line:l3} is obtained from Eq.~\ref{line:l2} since conditioned on the state $S_t$, emphatic trace $F_t$ computed with variables ``from the past'' is independent from the TD update using variables ``in the future''. Since the second expectation term in Eq.~\ref{line:l3} does not depend on the time step $t$ but only depends on the state value $s$ under the steady state distribution, we take it out of the limit and for clarity, we re-index its time step with a new variable $k$. This becomes
\begin{align}  
    & \sum_s \underbrace{d_{\mu} (s) \lim_{t\rightarrow\infty}  \E_{\mu} \left[F_t| S_t=s\right]}_{f(s)}\E_{\mu} \left[\vphi(S_k)  \sum_{i=k}^{k+n-1} \prod_{j=k}^{i-1}(\gamma_{j+1}\rho_j)\rho_i \left[\vphi(S_i) - \gamma \vphi(S_{i+1})\right]^T \middle\vert S_k=s\right]\label{line:l4}\\
    &= \sum_s f(s) \E_{\mu} \left[\vphi(S_k)  \sum_{i=k}^{k+n-1} \prod_{j=k}^{i-1}(\gamma_{j+1}\rho_j)\rho_i \left[\vphi(S_i) - \gamma \vphi(S_{i+1})\right]^T \middle\vert S_k=s\right]\label{line:l5}\\
    &= \mPhi^T \mF(\mI-\mP_{\pi}^n\mGamma^n) \mPhi,\label{line:l6}
\end{align}
where $\mF$ is a diagonal matrix with diagonal elements $f(s) \doteq d_{\mu}(s)\lim_{t\rightarrow\infty}\E_{\mu}[F_t | S_t=s]$, which we assume exists. Eq.~\ref{line:l6} reorganizes terms in Eq.~\ref{line:l5} into their matrices notations and forms the telescoping sum as in Eq.~\ref{eq:dev2}.

Recall that the sufficient condition for a matrix to be positive definite from \citet{sutton2016emphatic} is to have all positive diagonal entries and all negative off-diagonal entries in the key matrix, and that its row sum plus column sum is positive. The key matrix is $\mF(\mI-\mP_{\pi}^n\mGamma^n)$ where $\mF$ is a diagonal matrix with all positive entries on the diagonal, hence with discounts and transition probabilities smaller than 1, the key matrix has positive diagonal entries and all negative entries on the off-diagonal. The row sum is positive since $\mP_{\pi}^n$ is a transition matrix with row sum equal to 1 and the discounts are smaller than 1. Hence to make the key matrix $\mF(\mI-\mP_{\pi}^n\mGamma^n)$ into a positive definite matrix, we just need the columns sum to be positive. Let $\vf$ be the diagonal entries of $\mF$. If we define
\begin{align}
    \vf \doteq \left[\mI- (\mP_\pi^T)^n\mGamma^n\right]^{-1}\vd_{\mu},
\end{align}
where $\vd_{\mu}$ is the vector of diagonal elements of $\mD_\mu$, then the column sum of the key matrix is
\begin{align}
    \mathbf{1}^T \mF (\mI- \mP_{\pi}^n\mGamma^n) &= \vf^T  (\mI- \mP_{\pi}^n\mGamma^n)\\
    &= \vd_{\mu}^T \left[\mI- (\mP_\pi)^n\mGamma^n\right]^{-1} (\mI- \mP_{\pi}^n\mGamma^n)\\
    &= \vd_{\mu}^T,
\end{align}
which is an all positive vector. Therefore $\mF$ thus defined, the resulting key matrix is positive definite and the steady state learning updates are stable.

In order to apply the emphatic trace to every update, we need to derive the follow-on trace (NETD) at every time step that corresponds to the thus defined $\mF$ matrix. We show that this following trace gives the above defined $\mF$ matrix:
\[
F_t = \prod_{i=1}^n (\gamma_{t-i+1} \rho_{t-i}) F_{t-n} + 1, \text{ with }F_0, F_1, \ldots, F_{n-1}=1.
\]

Take the $n=2$ for an example, i.e. $F_t = \gamma_t\gamma_{t-1} \rho_{t-1}\rho_{t-2} F_{t-2} + 1, \text{ with }F_0, F_1=1$. Recall for any state $s$,
\begin{align}
\label{eq:why_netd}
    f(s) &\doteq d_{\mu}(s) \lim_{t\rightarrow \infty} \E_{\mu} [F_t \mid S_t=s]\\
    &= d_{\mu}(s) \lim_{t\rightarrow \infty} \E_{\mu} [\gamma_t\gamma_{t-1} \rho_{t-1}\rho_{t-2} F_{t-2} + 1 \mid S_t=s]\\
    &= d_{\mu}(s) + d_{\mu}(s) \lim_{t\rightarrow \infty} \E_{\mu} [\rho_{t-1}\rho_{t-2} \gamma_t\gamma_{t-1} F_{t-2} \mid S_t=s]\\
    &= d_{\mu}(s) + d_{\mu}(s) \sum_{a', s', a'', s''} P_{\mu}(S_{t-1}=s', A_{t-1}=a'|S_t=s) P_{\mu}(S_{t-2}=s'', A_{t-2}=a''|S_{t-1}=s') \cdot \notag\\
    &\hspace{7cm} \frac{\pi(a'|s')}{\mu(a'|s')}\frac{\pi(a''|s'')}{\mu(a''|s'')}\gamma(s)\gamma(s') \lim_{t\rightarrow \infty} \E_{\mu} [F_{t-2} \mid S_{t-2}=s'']\\
    &= d_{\mu}(s) + \cancel{d_{\mu}(s)} \sum_{a', s', a'', s''} \frac{\cancel{d_{\mu}(s')}\cancel{\mu(a'|s')}p(s|s', a')}{\cancel{d_{\mu}(s)}} \cdot \frac{\cancel{d_{\mu}(s'')}\cancel{\mu(a''|s'')}p(s'|s'', a'')}{\cancel{d_{\mu}(s')}} \cdot \notag \tag{by Bayes Rule}\\
    &\hspace{10cm} \frac{\pi(a'|s')}{\cancel{\mu(a'|s')}}\frac{\pi(a''|s'')}{\cancel{\mu(a''|s'')}} \frac{\gamma(s)\gamma(s')f(s'')}{\cancel{d_{\mu}(s'')}}\\
    &= d_{\mu}(s) + \gamma(s) \sum_{a', s', a'', s''} p(s|s', a')p(s'|s'', a'')\pi(a'|s')\pi(a''|s'') \gamma(s')f(s'')\\
    &= d_{\mu}(s) + \gamma(s) \sum_{s'}[\mP_{\pi}]_{s',s}\gamma(s')\sum_{s''}[\mP_{\pi}]_{s'',s'} f(s'').
\end{align}

Thus the vector
\begin{align}
\vf = \vd_{\mu} + {\mP^T_{\pi}}^2 \mGamma^2 \vf = (\mI + {\mP^T_{\pi}}^2 \mGamma^2 + {\mP^T_{\pi}}^4 \mGamma^4 + \cdots) \vd_{\mu} = (\mI- {\mP^T_{\pi}}^2 \mGamma^2)^{-1} \vd_{\mu}.
\end{align}
Since the case of any other positive integer value of $n$ can be derived exactly in the same way, we conclude
\begin{align}
    \mF = (\mI- {\mP^T_{\pi}}^n \mGamma^n)^{-1} \mD_{\mu}.
\end{align}

\section{Emphatic traces for the V-trace target}
\label{app:nevtrace_derivation}

Recall that V-trace was developed in \citep{espeholt2018} as a method to reduce the variance in importance sampling based off policy policy evaluation. The motivation for V-trace was to correct small off-policy discrepancies that result from parallelizing, i.e., the parameters of the actors lag behind the parameters of the learner. 
However, the analysis of V-trace was only performed in the tabular MDP setting and not with function approximation. We now show that with linear function approximation, V-trace suffers from the same stability issues as standard IS method, i.e., that the corresponding key matrix is not positive definite.  Recall that V-trace truncates the IS ratios in by some constant $\bar{\rho}$, such that $\bar{\rho}_t = \min \{\bar \rho, \rho_t\}.$ Before we begin we introduce some notation. 

We denote the denominator in Eq.~\ref{eq:vtrace_fp} by $\nu(s) = \sum_{a'\in \mathcal{A}}\min (\bar{\rho}\mu(a'|s), \pi(a'|s)).$ Using this notation, we have that the importance sampling ratio between the true target V-trace policy $\pi_{\bar{\rho}}$ (Eq.~\ref{eq:vtrace_fp}) and the behaviour policy $\mu$ is given by: $\rho^v_t = \frac{\pi_{\bar{\rho}}(A_t|S_t)}{\mu(A_t|S_t)} = \frac{\bar{\rho}_t}{\nu(S_t)}.$  

\subsection{WEVtrace} We define the emphatic trace for the V-trace update at $n=1$ as:
\begin{equation}
    \label{eq:FV-Trace}
    F^v_t = F^v_{t-1}\gamma_t \rho^v_{t-1} + 1, \,\forall t>0.
\end{equation}
where $F^v_0 = 1$. To see why it stabilizes V-trace learning, we now examine the limit of the $\mA$ matrix when using the truncated importance sampling ratios $\bar{\rho}_t$ as in V-trace together with the V-trace follow on trace (Eq.~\ref{eq:FV-Trace}). We have that:

\begin{align}
    \mA &= \lim_{t\rightarrow\infty} \E[\mA_t] = \lim_{t\rightarrow\infty} \E_{\mu}F^v_t\vphi(S_t) \bar{\rho}_t \left[\vphi(S_t) - \gamma_{t+1} \vphi(S_{t+1})\right]^T \\
    &= \sum_s d_{\mu} (s) \lim_{t\rightarrow\infty} \E_{\mu}\left[ F^v_t\bar{\rho}_t \vphi(S_t) \left[\vphi(S_t) - \gamma_{t+1} \vphi(S_{t+1})\right]^T \middle\vert S_t=s\right] \\
    &= \sum_s d_{\mu} (s) \lim_{t\rightarrow\infty} \E_{\mu}\left[ F^v_t \rho^v_t \nu(S_t) \vphi(S_t) \left[\vphi(S_t) - \gamma_{t+1} \vphi(S_{t+1})\right]^T \middle\vert S_t=s\right], \label{line1}
\end{align}   
plugging in the definition of $\rho^v_t$. Just like the derivation from Eq.~\ref{line:l3} to Eq.~\ref{line:l4}, we use the fact that given $S_t$, the emphatic trace $F^v_t$ is independent of $\rho_t^v \nu(S_t) \vphi_t (\vphi_t - \gamma \vphi_{t+1})$ and the expected value of the latter term does not depend on the time step $t$ under the steady state distribution, so for clarity we change the time index to a new variable $k$. Thus we have 
\begin{align}
    &\sum_s d_{\mu} (s) \lim_{t\rightarrow\infty} \E_{\mu}\left[ F^v_t \middle\vert S_t=s\right] \E_{\mu} \left[\rho^v_k \nu(S_k) \vphi(S_k) \left[\vphi(S_k) - \gamma_{k+1} \vphi(S_{k+1})\right]^T \middle\vert S_k=s\right] \label{line2}  \\
    &= \sum_s f^v (s) \nu(s) \E_{\mu} \left[ \rho^v_k  \vphi(S_k) \left[\vphi(S_k) - \gamma_{k+1} \vphi(S_{k+1})\right]^T \middle\vert S_k=s\right] \label{line3}\\
    &= \sum_s f^v (s) \nu(s) \E_{\pi_{\bar\rho}} \left[\vphi(S_k) \left[\vphi(S_k) - \gamma_{k+1} \vphi(S_{k+1})\right]^T \middle\vert S_k=s\right] \label{line4}\\
    &= \mPhi^T \mF^v \mN\left[\mI- \mP_{\pi_{\bar \rho}}\mGamma\right] \mPhi\label{line5}.
\end{align}
In Eq.~\ref{line3} we used the fact that $\nu(s)$ is a function of the state only (and not the action), in Eq.~\ref{line4} we replace the expectation over $\mu$ with an expectation over $\pi_{\bar \rho },$ and finally in Eq.~\ref{line5} we let $\mN$ be a diagonal matrix with elements $\nu(s)$ on the diagonal. 

It is easy to see that without the emphatic trace $F^v_t$, the V-trace steady state key matrix is $\mN \mD_\mu\left[\mI- \mP_{\pi_{\bar \rho}}\mGamma\right]$, which is not necessarily positive definite. Thus V-trace may suffer from instability issues with linear function approximation. Following \citep{sutton2016emphatic}, we have that $\mF^v = \left[\mI- \mP^T_{\pi_{\bar \rho}}\mGamma\right]^{-1}\mD_\mu$ with $F^v_t$ as defined in Eq.~\ref{eq:wetd}. Let's check that the key matrix $\mN \mF^v\left[\mI- \mP_{\pi_{\bar \rho}}\mGamma\right]$ is positive definite. First notice that the $\mF^v$ and $\mN$ are diagonal matrices with positive diagonal entries. With transition probabilities smaller than 1, the key matrix must have positive diagonal entries and negative off-diagonal entries. Moreover its row sum is an all positive vector. The column sum of the key matrix is
\begin{align}
    \mathbf{1}^T \mF^v (\mI- \mP_{\pi_{\bar \rho}}\mGamma) \mN  &= \vf^T (\mI- \mP_{\pi_{\bar \rho}}\mGamma) \mN \\
    &= \vd_{\mu}^T \left[\mI-\mP_{\pi_{\bar \rho}}\mGamma\right]^{-1} (\mI- \mP_{\pi_{\bar \rho}}\mGamma) \mN\\
    &= \vd_{\mu}^T \mN,
\end{align}
which is an all positive vector. Hence $F^v_t$ stabilized learning. Finally, combined with the definition of $\lambda^v_t$ in Eq.~\ref{eq:vtrace_lambda}, now we have the WEVtrace.

\subsection{NEVtrace} We define the emphatic trace for the $n$-step V-trace update as:
\begin{align}
\label{eq:nevtrace_def}
    F^{(n), v}_t &= \prod_{i=1}^n (\gamma_{t-i+1} \rho^v_{t-i}) F^{(n), v}_{t-n} + 1,
\end{align}

where $F^{(n), v}_0, F^{(n), v}_1, \ldots, F^{(n), v}_{n-1}=1$. The $\mA$ matrix for the emphatically modified V-trace update becomes
\begin{align}
    \mA &= \lim_{t\rightarrow\infty} \E[\mA_t] = \lim_{t\rightarrow\infty} \E_{\mu}F_t\vphi(S_t) \sum_{i=t}^{t+n-1} \prod_{j=t}^{i-1}(\gamma_{j+1}\bar{\rho}_j)\bar{\rho}_i \left[\vphi(S_i) - \gamma_{i+1} \vphi(S_{i+1})\right]^T \\
    &= \,\,\ldots \tag{similar derivations as Eq.~\ref{line:l1} to Eq.~\ref{line:l5}}\\
    &= \sum_s f^{(n), v}(s) \E_{\mu} \left[\vphi(S_k)  \sum_{i=k}^{k+n-1} \prod_{j=k}^{i-1}(\gamma_{j+1}\bar{\rho}_j)\bar{\rho}_i \left[\vphi(S_i) - \gamma_{i+1} \vphi(S_{i+1})\right]^T \middle\vert S_k=s\right]\\
    &= \sum_{i=k}^{k+n-1}\sum_s f^{(n), v}(s) \E_{\mu} \left[\vphi(S_k) \prod_{j=k}^{i-1} (\gamma_{j+1}\rho^v_j \nu(S_j)) \rho^v_i \nu(S_i)\left[\vphi(S_i) - \gamma_{i+1} \vphi(S_{i+1})\right]^T \middle\vert S_k=s\right]\\
    &= \sum_{i=k}^{k+n-1} \mPhi^T \mN^{i-k+1} \mF^{(n), v} \left[\mP_{\pi_{\bar \rho}}^{i-k}\mGamma^{i-k}- \mP_{\pi_{\bar \rho}}^{i-k+1}\mGamma^{i-k+1}\right] \mPhi\\
    &\approx \sum_{i=k}^{k+n-1} \mPhi^T \mF^{(n), v} \left[( \mN^{i-k} \mP_{\pi_{\bar \rho}}^{i-k}\mGamma^{i-k}- \mN^{i-k+1} \mP_{\pi_{\bar \rho}}^{i-k+1}\mGamma^{i-k+1})\right] \mPhi\label{eq:approx_nevtrace}\\
    &= \mPhi^T \mF^{(n), v} (\mI- \mN^n \mP_{\pi_{\bar \rho}}^n\mGamma^n) \mPhi, 
\end{align}
where $f^{(n), v}(s)$ are diagonal entries of $\mF^{(n), v}$ and $\mF^{(n), v} \doteq \left[\mI- \mN^n {\mP^T_{\pi_{\bar \rho}}}^n\mGamma^n\right]^{-1}\mD_\mu$. Similar to before, this $\mF^{(n), v}$ makes the approximate key matrix $\mF^{(n), v} (\mI- \mN^n \mP_{\pi_{\bar \rho}}^n\mGamma^n)$ positive definite. Recall that $\mN$ is a diagonal matrix with either value 1 or some value in $(0, 1)$ for states where the IS weights are clipped by $\bar{\rho}$, and Eq.~\ref{eq:approx_nevtrace} is approximate by treating one of the $\mN$ matrices as an identity matrix.  

To see why this follows from the NEVtrace definition in Eq.~\ref{eq:nevtrace_def}, recall derivations for NETD trace in Eq.~\ref{eq:why_netd} for $n=2$. Here we have
\begin{align}
    f^{(n), v}(s) &= d_{\mu}(s) \lim_{t\rightarrow \infty} \E_{\mu} [F^{(n), v}_t \mid S_t=s]\\
    &= \,\, \ldots\\
    &= d_{\mu}(s) + \gamma(s)\sum_{s'}[\mP_{\pi_{\bar \rho}}]_{s,s'}\nu(s')\gamma(s')\sum_{s''}[\mP_{\pi_{\bar \rho}}]_{s',s''}\nu(s'') f(s'').
\end{align}
Thus the vector 
\begin{align}
\vf^{(n), v} = \vd_{\mu} + \mN^2 {\mP^T_{\pi_{\bar \rho}}}^2 \mGamma^2 \vf = (\mI + \mN^2 {\mP^T_{\pi_{\bar \rho}}}^2 \mGamma^2 + \mN^4 {\mP^T_{\pi_{\bar \rho}}}^4 \mGamma^4 + \cdots) \vd_{\mu} = (\mI- \mN^2 {\mP^T_{\pi_{\bar \rho}}}^2 \mGamma^2)^{-1} \vd_{\mu}.
\end{align}
Extending this to any other positive integer value of $n$, we conclude
\begin{align}
    \mF^{(n), v} = (\mI- \mN^n {\mP^T_{\pi_{\bar \rho}}}^n \mGamma^n)^{-1} \mD_{\mu}.
\end{align}

\newpage
\section{Hyperparameters}
\label{sec:surreal_hyperparams}

\textbf{Architectures.}

\begin{table}[h!]
\caption{Network architecture}
\begin{center}
\begin{tabular}{|l|l|l|l|}
    \hline
    Parameter & \\
    \hline 
    convolutions in block & (2, 2, 2, 2) \\
    channels & (64, 128, 128, 64) \\
    kernel sizes & (3, 3, 3, 3) \\
    kernel strides & (1, 1, 1, 1)  \\
    pool sizes & (3, 3, 3, 3)  \\
    pool strides & (2, 2, 2, 2)  \\
    frame stacking & 4 \\
    head hiddens & 512  \\
    activation & Relu \\

    \hline
\end{tabular}
\end{center}
\label{table:hyperparameters2}
\end{table}

Our DNN architecture is composed of a shared torso, which then splits to different heads. We have a head for the policy and a head for the value function (multiplied by the number of auxiliary tasks). Each head is a two-layered MLP with 512 hidden units, where the output dimension corresponds to 1 for the vale function head. For the policy head, we have $|\mathcal{A}|$ outputs that correspond to softmax logits. We use ReLU activations on the outputs of all the layers besides the last layer. For the policy head, we apply a softmax layer and use the entropy of this softmax distribution as a regularizer.  

The \textbf{torso} of the network is composed from residual blocks. In each block there is a convolution layer, with stride, kernel size, channels specified in \cref{table:hyperparameters2}, with an optional pooling layer following it. The convolution layer is followed by n - layers of convolutions (specified by blocks), with a skip contention. The output of these layers is of the same size of the input so they can be summed. The block convolutions have kernel size $3,$ stride $1$.  

\textbf{Hyperparameters.}
\cref{table:hyperparameters} lists all the hyperparameters used by our agent. Most of the hyperparameters follow the reported parameters from the IMPALA paper. For completeness, we list all of the exact values that we used below.  
\begin{table}[h!]
\caption{Hyperparameters table}
\begin{center}
\begin{tabular}{|l|l|l|l|}
    \hline
    Parameter & Value  \\
    \hline 
    total environment steps & 200e6 \\
    optimizer & RMSPROP \\
    start learning rate & $6 \cdot 10^{-4}$ (mixed), $2 \cdot 10^{-4}$ (fixed) \\
    end learning rate & 0  \\
    decay & 0.99 \\
    eps & 0.1 \\
    importance sampling clip & 1 \\
    gradient norm clip & $0.3$ (mixed), $1$ (fixed)\\
    trajectory $n$ & $40$ (mixed), $10$ (fixed)\\
    batch size (m) & 18 \\
    discount  $\gamma$ (main) &  $\sigmoid(4.6)\approx.99$ \\
    discount  $\gamma^1$ ($1^{st}$ auxiliary) &  $\sigmoid(4.4)\approx.988$ \\
    discount  $\gamma^2$ ($2^{nd}$ auxiliary) &  $\sigmoid(4.2)\approx.985$ \\
    \hline
\end{tabular}
\end{center}
\label{table:hyperparameters}
\end{table}

\newpage
\section{Atari experiment results}
\label{sec:atari_results}

\begin{figure}[h!]
\centering
\subfigure[NETD-ACE on Surreal versus Surreal]{
\includegraphics[width=0.59\columnwidth]{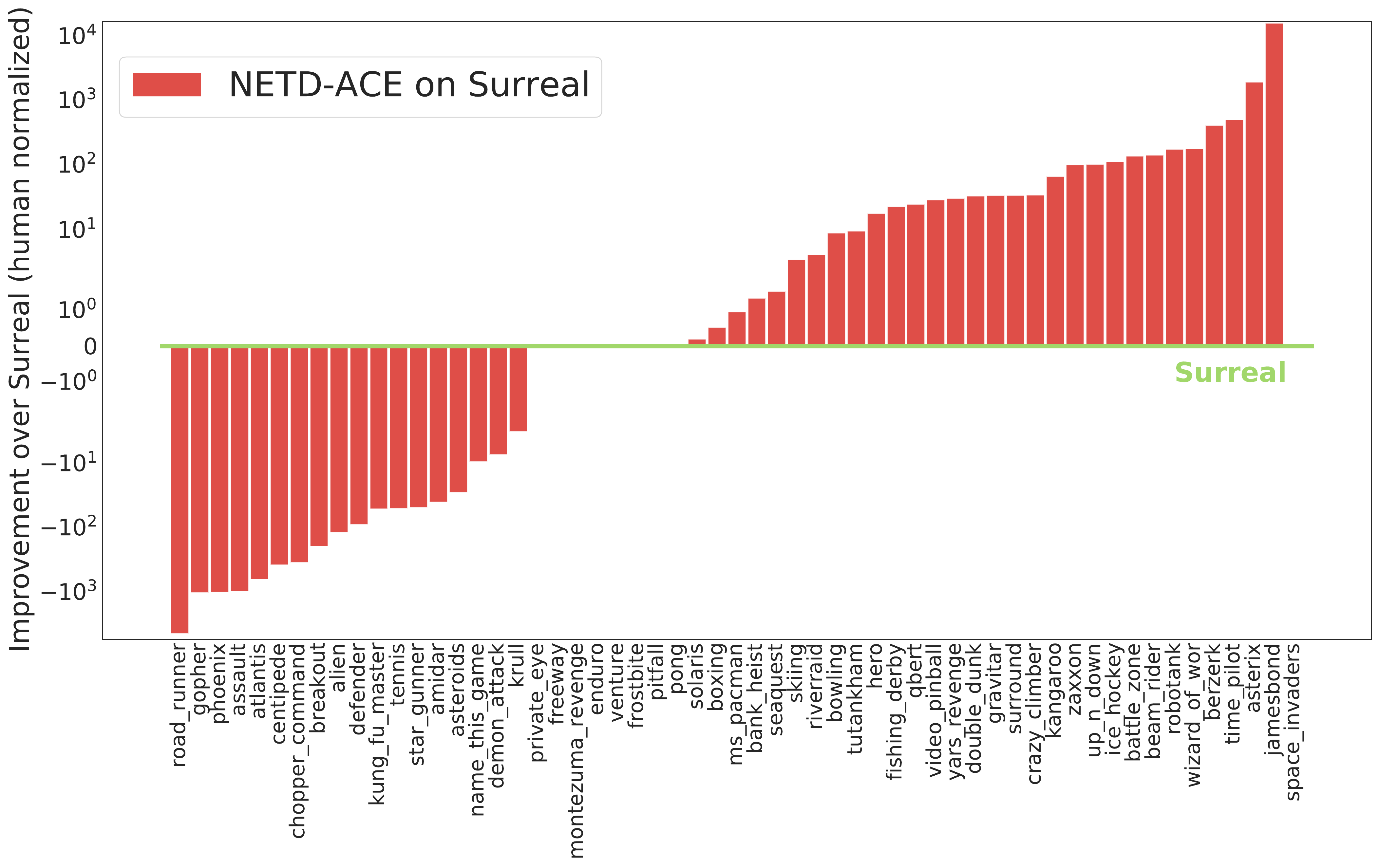}
\label{fig:barplot_NETD_ACE}
}
\vspace{-0.3cm}
\subfigure[NETD on Surreal versus Surreal]{
\includegraphics[width=0.59\columnwidth]{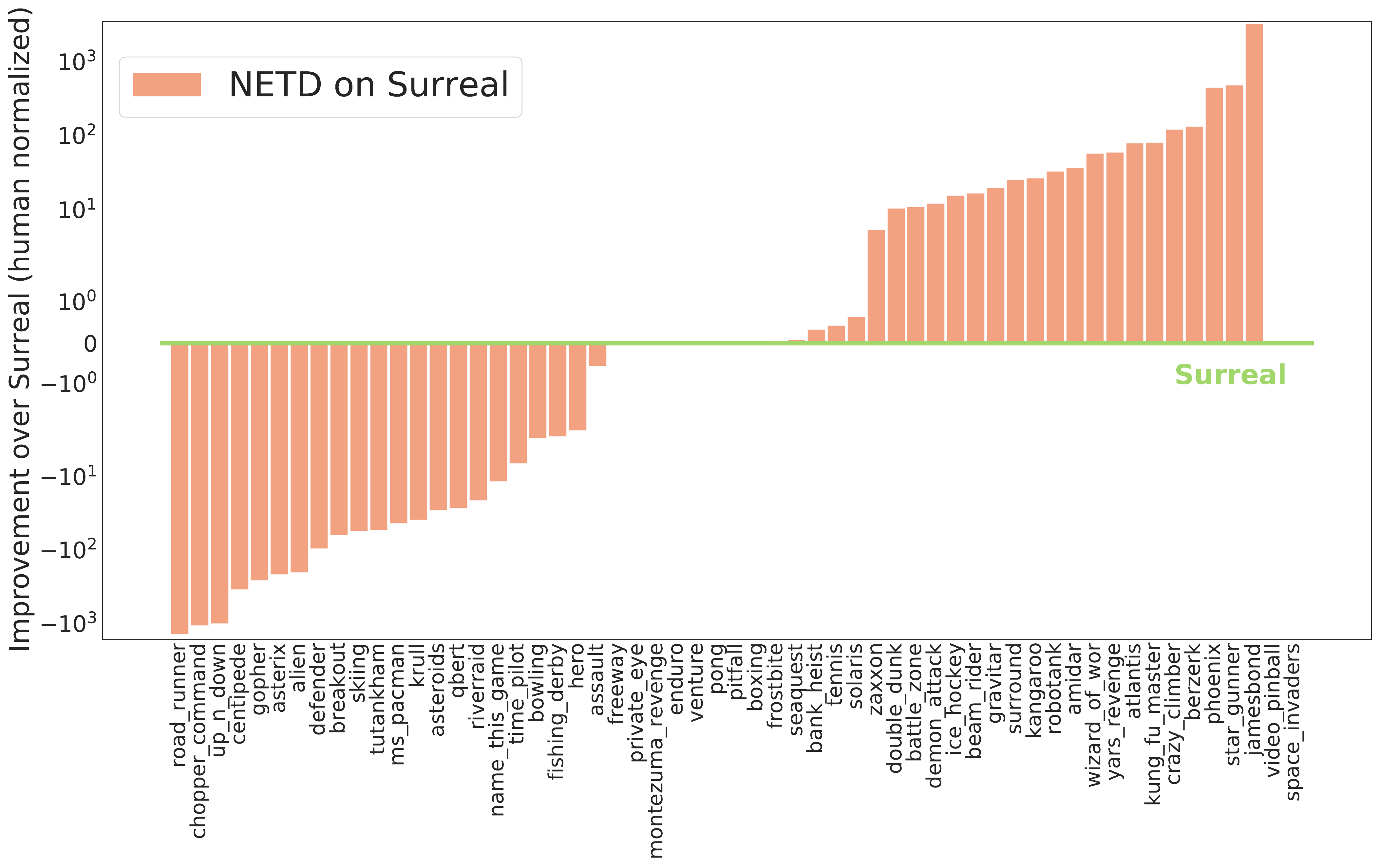}
\label{fig:barplot_NETD}
}
\vspace{-0.3cm}
\subfigure[NEVtrace on Surreal versus Surreal]{
\includegraphics[width=0.59\columnwidth]{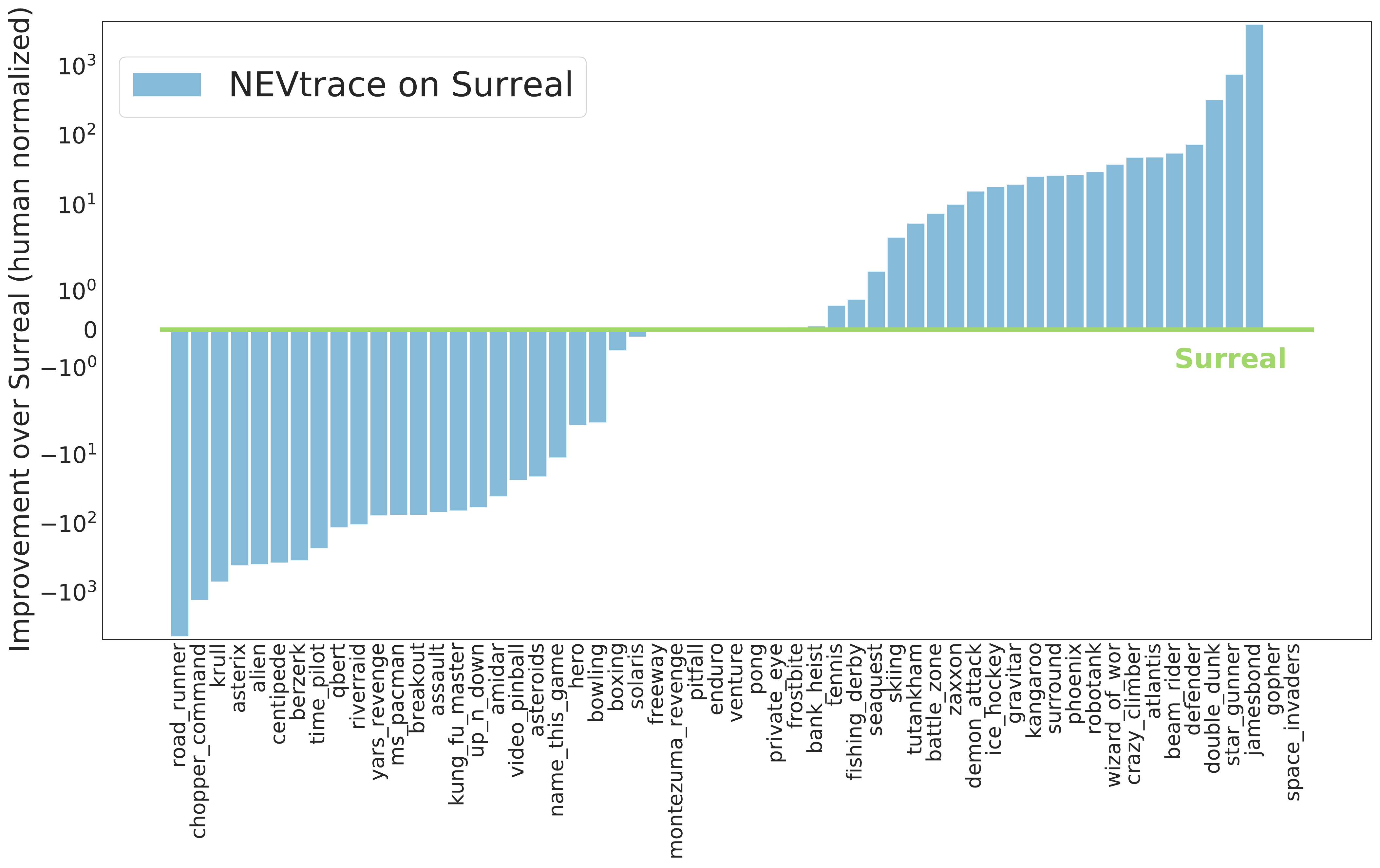}
\label{fig:barplot_NEVtrace}
}
\caption{
Improvement in individual human normalized game scores compared to Surreal: \subref{fig:barplot_NETD_ACE} NETD-ACE, \subref{fig:barplot_NETD} NETD and \subref{fig:barplot_NEVtrace} NEVtrace on Surreal versus Surreal in green. Results averaged across 3 seeds and the evaluation phase.}
\label{fig:barplot}
\end{figure}

\newpage
\section{Diagnostic experiments}
We include the full experiment results and additional implementation details on the diagnostic MDPs in this section. 

\subsection{Two-state MDP}
\label{app:2state}
In Fig.~\ref{fig:2state_n_0}, the rightmost column shows three baselines: Off-policy TD(0), Clipped off-policy $n$-step TD where all IS weights are clipped except for those directly on the TD error, i.e. unbiased V-trace by \cite{espeholt2018}, and V-trace. All three baselines diverged on this MDP. The emphatic algorithms all converged to the optimal fixed point $\theta=0$, however, notice that both emphatic TD (NETD/WETD) and emphatic V-trace (NEVtrace/WEVtrace) algorithms exhibited unstable learning, with some runs experiencing large jumps in value error late in training. The clipped emphatic traces (IS clipped at 1) theoretically have a finite variance, and empirically enjoy a faster convergence and stable learning after initial fluctuations.

As we increase the bootstrap length to $n=5$, the Off-policy 5-step TD baseline converged quickly while the other two baselines Clipped off-policy $n$-step TD and V-trace exhibited higher variances (see Fig.~\ref{fig:2state_n_4}). Similar to before, the clipped traces (row 2) were effective in variance reduction and demonstrated fast convergence. In comparison, WETD exhibited more variance in learning, NEVtrace converged slowly and NETD, WEVtrace were unstable late in training. 

\begin{figure*}[h!]
\centering
\includegraphics[width=\linewidth]{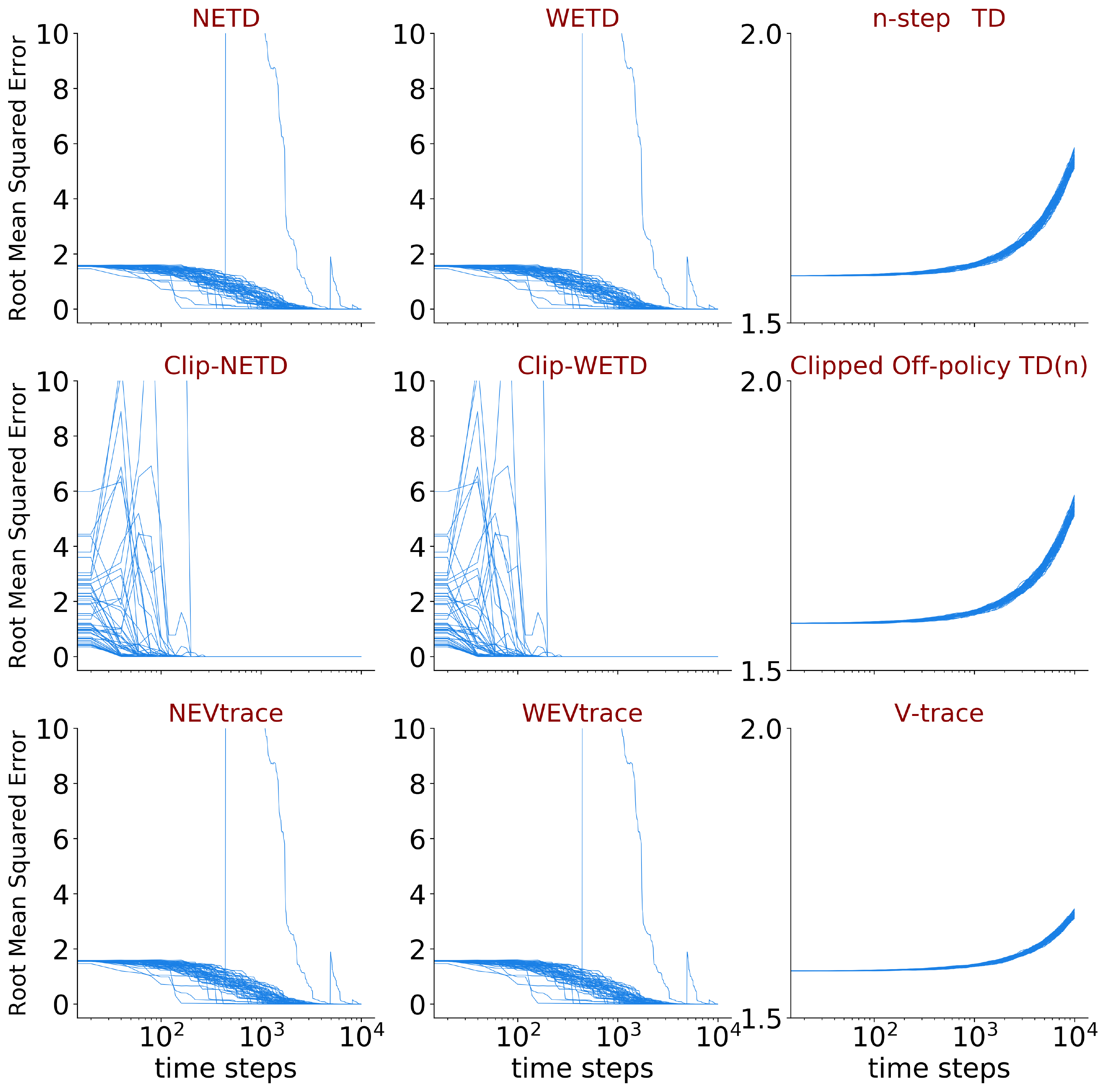}
\caption{RMSE over time on the two-state MDP with $\gamma=0.9$ and $n=1$. Each subplot shows fifty independent runs of each algorithm using the best setting setting of $\alpha$ found in the hyperparameter sweep. Note the log scale on x-axis.}
\label{fig:2state_n_0}
\end{figure*}

\begin{figure*}[h!]
\centering
\includegraphics[width=\linewidth]{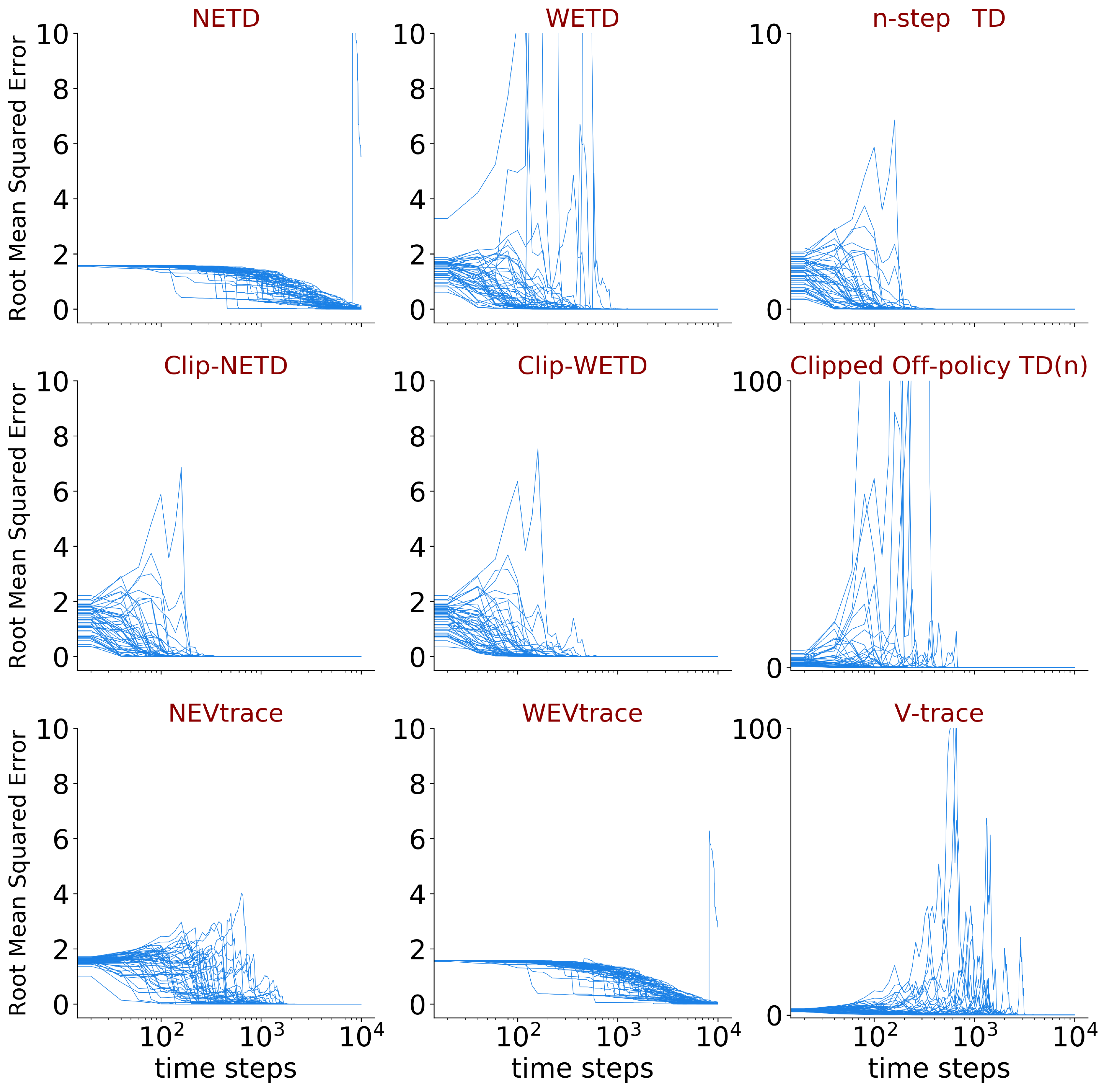}
\caption{Same experiment setup on the two-state MDP as Fig.~\ref{fig:2state_n_0} except $n=5$. Note the log scale on x-axis.}
\label{fig:2state_n_4}
\end{figure*}

\subsection{Collision Problem}
\label{app:collision}
We present the full results in Fig.~\ref{fig:collision_learning_curve} and Fig.~\ref{fig:collision_u_curves}. All algorithms achieved stable learning with their best learning rates from hyperparameter sweeps. Emphatic algorithms consistently achieved the smallest mean RMSE averaged over 200 runs for all values of $n$ tested. 

\begin{figure*}[h!]
\centering
\includegraphics[width=\linewidth]{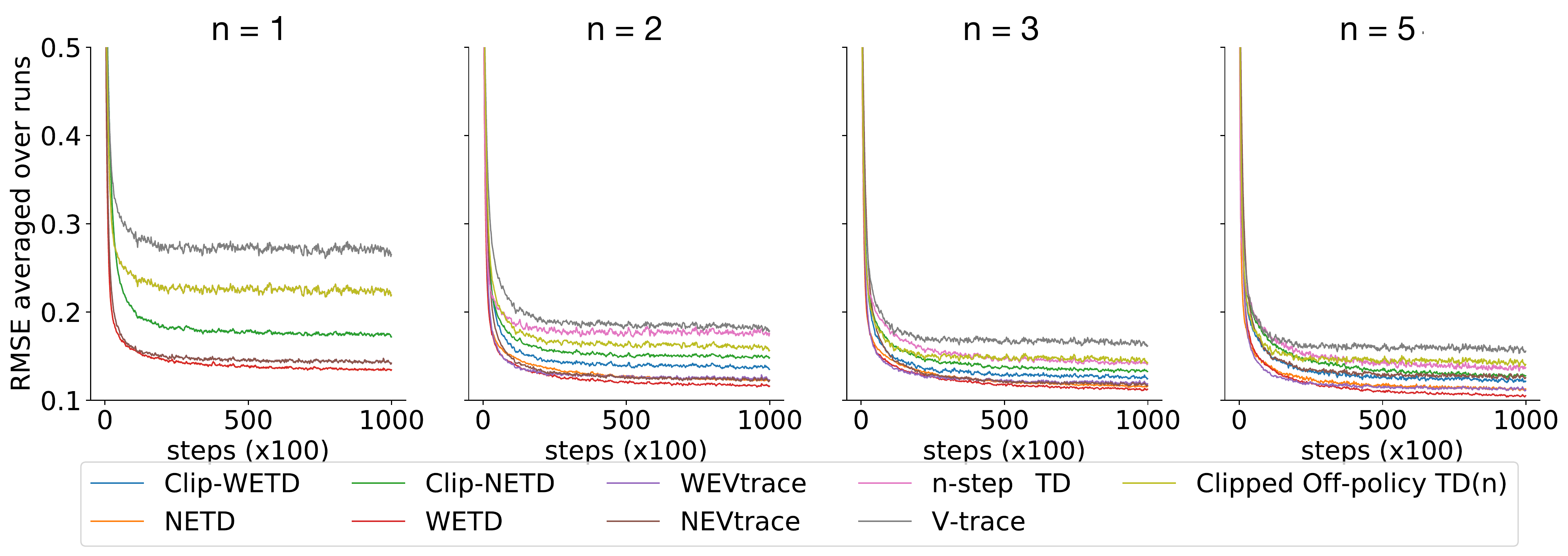}
\caption{Learning curves comparison on the Collision Episodic MDP. The results are averaged over 200 independent runs with each algorithms best hyper-parameter setting from the sweep. The three baselines produced the highest averaged RMSE errors. Emphatic algorithms ETD, NETD and EVtrace, NEVtrace consistently produced the lowest RMSEs for all bootstrap values $n$, followed by the clipped emphatic traces.}
\label{fig:collision_learning_curve}
\end{figure*}

\begin{figure*}[h!]
\centering
\includegraphics[width=\linewidth]{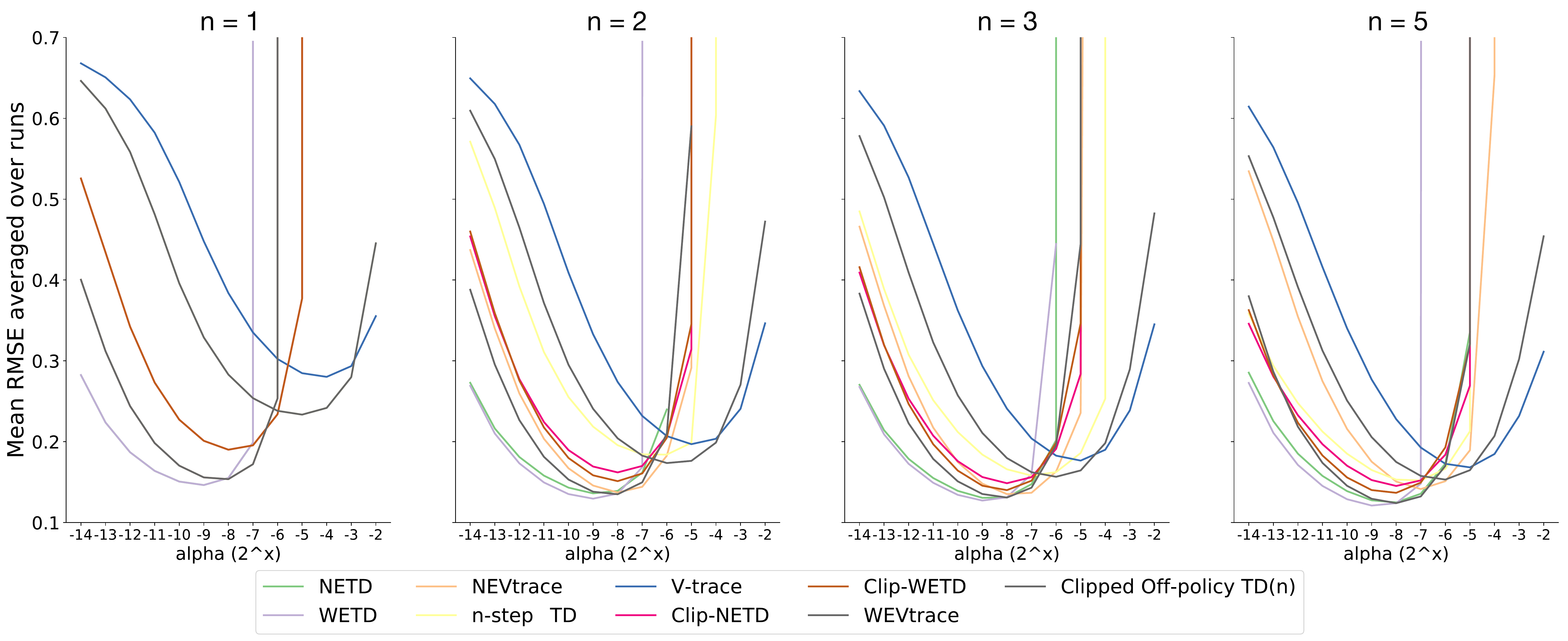}
\caption{Hyper-parameter sensitivity comparison on the Collision Episodic MDP. Each data point in the plot portrays the mean RMSE averaged over 200 runs for varying learning rate $\alpha$ and the bootstrap length $n$. The emphatic algorithms achieved the lowest value error with smaller learning rates than the baselines.}
\label{fig:collision_u_curves}
\end{figure*}

\subsection{Baird's counterexample}
\label{app:bairds}

Baird's counterexample (Fig.~\ref{fig:baird_mdp}) is a simple MDP with has seven states and simple linear features that causes TD and other methods to diverge. The features are designed to cause unnecessary generalization, even though the true value function is perfectly representable. This over-parameterization combined with a large mismatch in the target and behavior policies typically causes divergence. See \citep{sutton2018} for an extensive discussion and analysis of Baird's counterexample. 

Using the TD(0) learning update (Fig.~\ref{fig:baird_n_0}), both emphatic traces (row 1) converged quickly, with occasional instability in some runs. The clipped emphatic traces (row 2) exhibit slow learning, but exhibit a clear downward trend. The n-step TD baselines and all methods with V-trace targets diverged. This is not surprising as Baird's counterexample is considerably harder than the two-state MDP---Sutton's ETD($\lambda$) diverges on Baird's counterexample, but converges on the two-state MDP \cite{sutton2016emphatic}. Using 5-step TD learning (Fig.~\ref{fig:baird_n_4}) improves the performance of several methods. The WETD algorithms performed poorly compared to NETD algorithms and the n-step TD baseline. We see the effect of IS clipping: lowering variance of both WETD and NETD. Vtrace, NEVtrace, and WEVtrace all slowly diverged and WEVtrace exhibited unstable learning late in training. Clipped emphatic methods and n-step TD all benefit from longer n-step targets, significantly improving over their one-step variants in Fig ~\ref{fig:baird_n_0}. In this challenging MDP, one-step methods are not sufficient for fast and stable learning.  

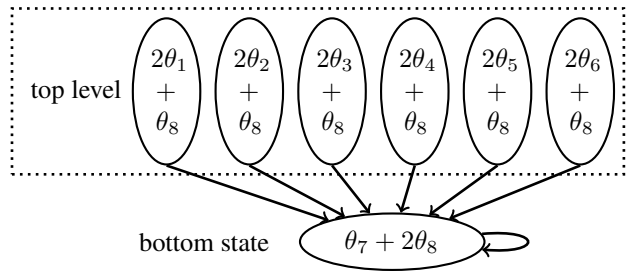
\begin{figure}[h!]
\centering
\begin{tikzpicture}[dgraph]
\node[ellip] (s1) at (0.,1) {$2\theta_1$\\$\,+$\\$\,\theta_8$};
\node[ellip] (s2) at (1.1,1) {$2\theta_2$\\$\,+$\\$\,\theta_8$};
\node[ellip] (s3) at (2.2,1) {$2\theta_3$\\$\,+$\\$\,\theta_8$};
\node[ellip] (s4) at (3.3,1) {$2\theta_4$\\$\,+$\\$\,\theta_8$};
\node[ellip] (s5) at (4.4,1) {$2\theta_5$\\$\,+$\\$\,\theta_8$};
\node[ellip] (s6) at (5.5,1) {$2\theta_6$\\$\,+$\\$\,\theta_8$};
\node[ellip,text width=1.5cm] (s7) at (3,-1) {$\,\,\theta_7+2\theta_8$};

\draw[](s1.south)--(s7);
\draw[](s2.south)--(s7);
\draw[](s3.south)--(s7);
\draw[](s4.south)--(s7);
\draw[](s5.south)--(s7);
\draw[](s6.south)--(s7);
\draw[](s7.5)to[out=5, in=-5,looseness=10] (s7.-5);

\node[input] (ss) at (-1.2, 1) {top level};
\node[system,fit=(ss) (s1) (s2) (s3) (s4) (s5) (s6)] {};
\node[input] (tt) at (0.5, -1) {bottom state};
\end{tikzpicture}
\caption{Baird's counterexample MDP. Solid lines indicate the target policy $\pi($down$|\cdot)=1$, ending up in the bottom state. The behavior policy $\mu($up$|\cdot)=6/7, \mu($down$|\cdot)=1/7$. When action is ``up'', the agent goes to a random state on the top level. When action is ``down'', the agent goes to the bottom state.}
\label{fig:baird_mdp}
\end{figure}

\begin{figure*}[h!]
\centering
\includegraphics[width=\linewidth]{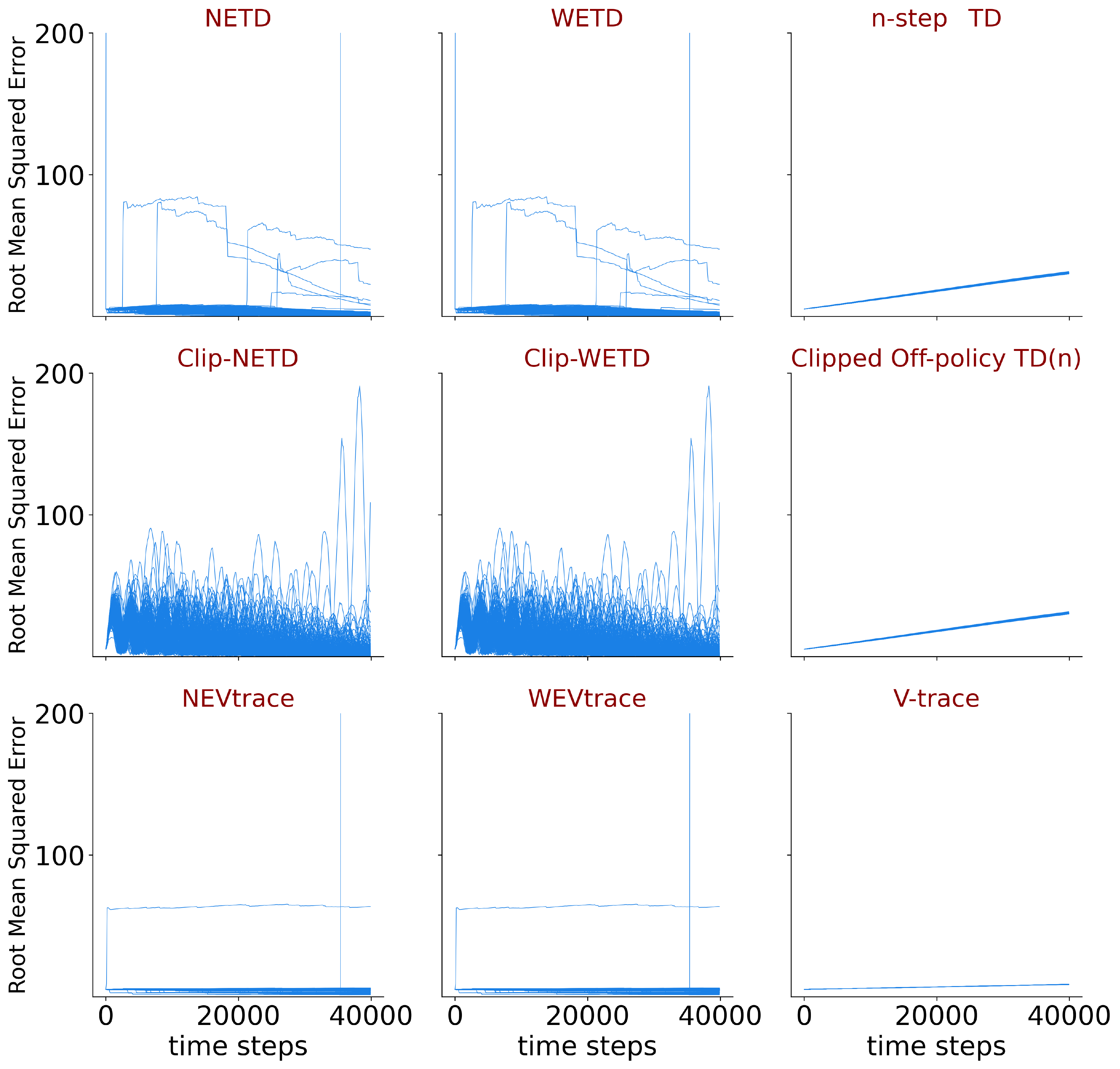}
\caption{Baird's MDP, $n=1, \gamma=0.9$, 200 indepdendent runs. Each algorithm was run with its best hyperparameter setting. Note the log scale on x-axis.}
\label{fig:baird_n_0}
\end{figure*}

\begin{figure*}[h!]
\centering
\includegraphics[width=\linewidth]{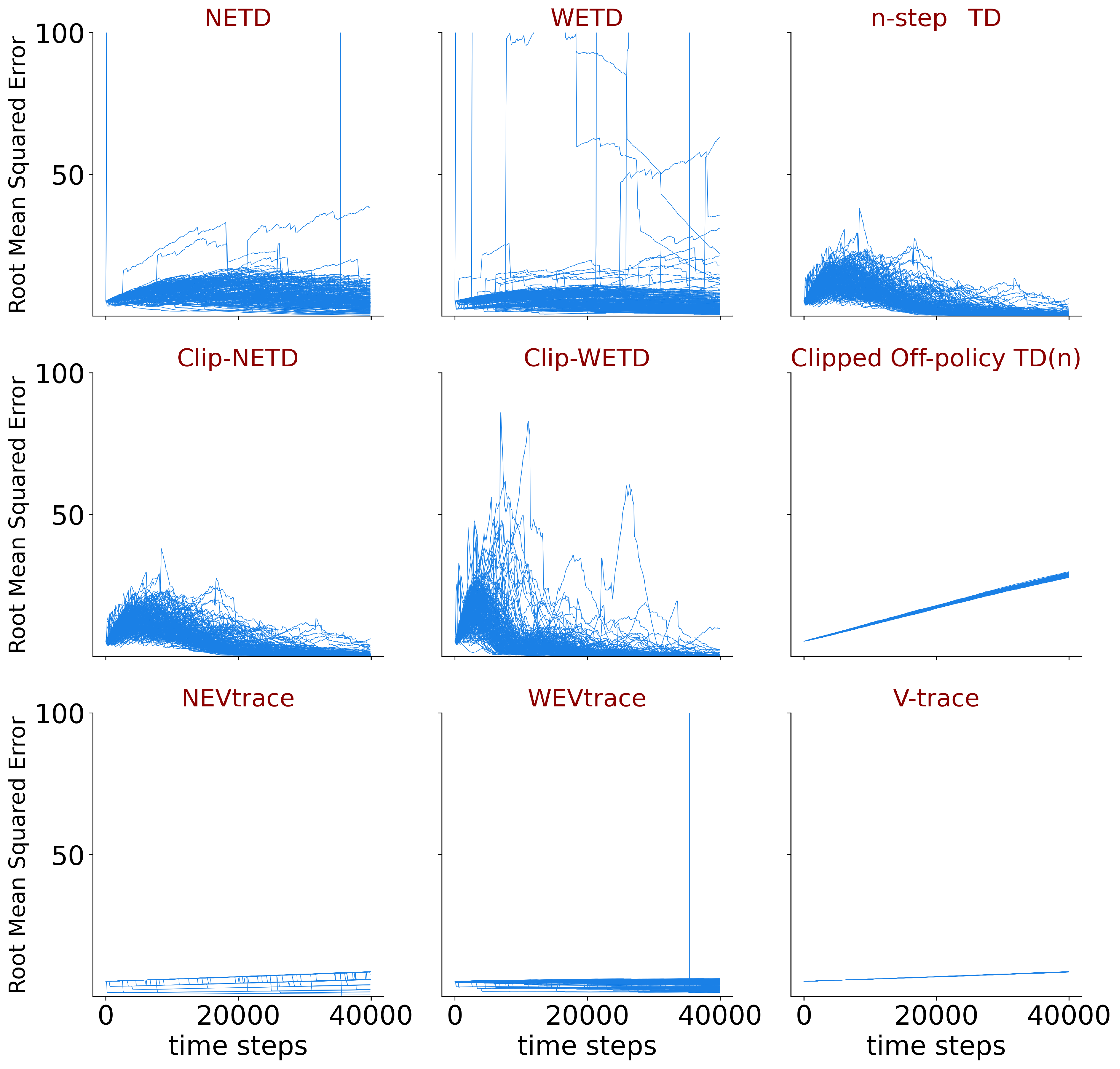}
\caption{Baird's MDP, $n=5, \gamma=0.9$, 200 indepdendent runs. Each algorithm was run with its best hyperparameter setting. NETD has a smaller variance than ETD algorithms. Clipping the IS in emphatic traces help reduce the variance. Note the log scale on x-axis.}
\label{fig:baird_n_4}
\end{figure*}

\end{document}